\documentclass[twoside]{article}

\usepackage[accepted]{aistats2024}
\usepackage[round]{natbib}

\newenvironment{claim}[1]{\par\noindent\underline{Claim:}\space#1}{}
\newenvironment{claimproof}[1]{\par\noindent\underline{Proof:}\space#1}{\hfill $\square$} % 

\usepackage{booktabs}
\usepackage{amssymb, amsmath, mathtools, amsthm} %for mathematics
\usepackage{graphicx, float} %for figures
\usepackage{chngcntr} %for numbering of equations and figures
\usepackage{hyperref} %for references to link
\usepackage{bbm} %for indicator function
\usepackage{indentfirst} %for indenting the first paragraph in a chapter
\usepackage{microtype} %for reformatting short lines
\usepackage{tikz}
\usetikzlibrary{matrix}
\usepackage{graphicx}
\usepackage{booktabs} % http://ctan.org/pkg/booktabs
\usepackage{xparse}   % http://ctan.org/pkg/xparse
\NewDocumentCommand{\rot}{O{45} O{1em} m}{\makebox[#2][l]{\rotatebox{#1}{#3}}}%
\usepackage{tabularx}

\usepackage{xcolor} %for colouring lines
\usepackage{subcaption}
\usepackage{enumerate}
\usepackage{enumitem}
\usepackage{comment}
\usepackage{thmtools} %for restatable command to repeat theorems
\usepackage{algorithm, algorithmicx, algpseudocode}
\usepackage{cuted}
\usepackage{arydshln}
\usepackage{bm}
\usepackage{dsfont}
\usepackage{ulem} % for \sout
\usepackage{subcaption}
\usepackage{booktabs}
\usepackage{placeins}

\graphicspath{ {../figures/} }

\renewcommand{\P}{\mathbb{P}} %probability
\newcommand{\Q}{\mathbb{Q}} %probability
\newcommand{\Prob}{\P}
\newcommand{\E}{\mathbb{E}} %expectation
\newcommand{\Var}{\mathbb{V}} %variance
\newcommand{\VH}{\mathbb{V}_{\mathcal{H}}} %variance in H
\newcommand{\MMD}{\mathrm{MMD}} % MMD
\newcommand{\1}{\mathds{1}} %indicator function
\newcommand{\R}{\mathbb{R}} %real numbers
\newcommand{\N}{\mathbb{N}} %natural numbers
\renewcommand{\O}{\mathcal{O}} %big O

\DeclarePairedDelimiter\norm{\lVert}{\rVert} %makes sure subscripts of the norms are placed properly
\newcommand\myeq{\mkern2.5mu{=}\mkern2.5mu} %equal sign with small spacing
\renewcommand\mid{\mkern4mu{|}\mkern4mu} %equal sign with small spacing
\newcommand{\defeq}{\overset{\textrm{def}}{=}}

\renewcommand{\H}{\ensuremath{{\mathcal H}}}

\newcommand{\rset}{\mathbb{R}}

\newcommand{\VI}{\mathrm{I}}

\newcommand{\forestass}[1]{\textbf{(F#1)}}
\newcommand{\dataass}[1]{\textbf{(D#1)}}
\newcommand{\kernelass}[1]{\textbf{(K#1)}}

\newcommand{\Ybf}{\mathbf{Y}}
\newcommand{\Yibf}{\mathbf{Y}_i}
\newcommand{\ybf}{\mathbf{y}}

\newcommand{\Xbf}{\mathbf{X}}
\newcommand{\xbf}{\mathbf{x}}

\newcommand{\Zbf}{\mathbf{Z}}
\newcommand{\Zcal}{\mathcal{Z}}
\newcommand{\PYgX}{\P_{\Ybf \mid \Xbf}}
\newcommand{\PYgXmj}{\P_{\Ybf \mid \Xbf^{(-j)}}}
\newcommand{\PYgx}{\P_{\Ybf \mid \Xbf=\xbf}}
\newcommand{\PYgxmj}{\P_{\Ybf \mid \Xbf^{(-j)} = \xbf^{(-j)}}}

\newcommand{\PY}{\P_{\Ybf}}

\newcommand{\hmun}{\mu_n(\xbf)}
\newcommand{\hmunX}{\mu_n(\Xbf)}

\newcommand{\hmunbruteXl}{{\mu}_n(\Xbf_i'^{(-j)})}
\newcommand{\hmunbruteX}{{\mu}_n(\Xbf^{(-j)})}

\newcommand{\hmunprojXl}{{\mu}_n^{(-j)}(\Xbf'_i)}
\newcommand{\hmunprojX}{{\mu}_n^{(-j)}(\Xbf)}
\newcommand{\hmunprojx}{{\mu}_n^{(-j)}(\xbf)}
\newcommand{\Kbf}{\mathbf{K}}
\newcommand{\kbf}{\mathbf{k}}

\newcommand{\Lcal}{\mathcal{L}}
\newcommand{\wbf}{\mathbf{w}}

\newcommand\blueft[1]{{\boldsymbol{\textcolor{blue}{#1}}}}

\newtheorem{corollary}{Corollary}
\newtheorem{lemma}{Lemma}

\begin{document}

\twocolumn[

\aistatstitle{MMD-based Variable Importance for Distributional Random Forest}

\aistatsauthor{Clément Bénard $^{1*}$ \And Jeffrey Näf $^{2*}$ \And  Julie Josse $^2$}

\aistatsaddress{ \\[-0.5em] $^1$Safran Tech, Digital Sciences \& Technologies, 78114 Magny-Les-Hameaux, France \\ $^2$Inria, PreMeDICaL Team, University of Montpellier
\\ $^*$ Equal contribution}]

\begin{abstract}
Distributional Random Forest (DRF) is a flexible forest-based method to estimate the full conditional distribution of a multivariate output of interest given input variables. In this article, we introduce a variable importance algorithm for DRFs, based on the well-established drop and relearn principle and MMD distance.
While traditional importance measures only detect variables with an influence on the output mean, our algorithm detects variables impacting the output distribution more generally.
We show that the introduced importance measure is consistent, exhibits high empirical performance on both real and simulated data, and outperforms competitors. In particular, our algorithm is highly efficient to select variables through recursive feature elimination, and can therefore provide small sets of variables to build accurate estimates of conditional output distributions.
\end{abstract}

\section{INTRODUCTION} \label{sec:intro}

\paragraph{Context and Objectives.}
Distributional Random Forest (DRF) \citep{DRF-paper} is an efficient algorithm designed for estimating the conditional distribution of target outputs given a set of input variables.  It is inspired from the highly popular Random Forest algorithm proposed by \citet{breiman2001random}, which has found widespread use in both classification and regression problems.  Unlike Breiman's forests, which provide only the conditional output mean, DRF goes a step further by offering the complete conditional output distribution. This capacity enables users to compute a wide range of quantities of interest with a high accuracy and low computational cost in a subsequent step.  These computations encompass the calculation of conditional quantiles, assessments of conditional independence, evaluation of conditional copulas, and the estimation of heterogeneous treatment effects.
The main features of DRF are the Maximum Mean Discrepancy (MMD) \citep{gretton2007kernel} used as splitting criterion, and the adaptive nearest neighbor formulation of Random Forest \citep{lin2006random, scornet2016random}.
Unfortunately, DRF also inherits the black-box nature of forest-based methods. Indeed, the large number of operations involved in their prediction mechanisms makes it impossible to grasp how inputs are combined to generate predictions. This lack of interpretability is a strong limitation, in particular for applications with critical decisions at stake, such as healthcare. Therefore, the pursuit of interpretability for black-box algorithms has gained considerable momentum in the machine learning community in recent years, with variable importance measures emerging as one of the main post-hoc method for elucidating these complex models.

The principle of variable importance measures is to quantify the influence of each input variable in a given prediction task. 
Two importance measures were originally proposed along with Breiman's forests: the Mean Decrease Accuracy (MDA) \citet{breiman2001random}, and the Mean Decrease Impurity (MDI) \citet{breiman2003atechnical}. The principle of the MDA is to permute the values of a given input variable to break its relation with the output, and compute the associated decrease of accuracy of the forest, defining the importance value for the permuted variable. Although this approach is widely used because of its intuitive definition and small computational cost, many empirical and theoretical studies have shown a strong bias when inputs are correlated \citep{strobl2008conditional, auret2011empirical, gregorutti2017correlation, hooker2019please, Sobol_MDA}. On the other hand, the MDI is defined as the sum of weighted decreases of impurity over all tree nodes that split on a given variable. However, \citet{strobl2007bias} highlight several practical flaws of the MDI, and \citet{scornet2023trees} shows that the MDI is ill-defined, except in restricted settings. Overall, there is a growing consensus in the machine learning community that other alternatives should be preferred to quantify variable importance for Random Forest. 
Regarding non-parametric multi-dimensional output regression, there appears to be little literature on variable importance. Recently, \citet{MRFVI} develop an importance measure for the multivariate Random Forest (MRF) of \citet{segal2011multivariate}, based on split improvement criteria. While this is an important first step, this measure suffers from similar limitations as the original MDI, mentioned above. Additionally, the \textit{drf} package \citep{drf-package} also provides an importance measure, defined as the frequency of splits involving a given variable, following the proposal of \citet{athey2019generalized} for generalized forests, and denoted by vimp-drf throughout the article. This approach is also purely empirical and does not provide a precise quantification of the impact of inputs on the output distribution.

Instead of empirical definitions, variable importance should be first defined through theoretical quantities, and then estimated in a second step with appropriate algorithms, as argued by \citet{Williamson} and \citet{Sobol_MDA} for example. In particular, the drop and relearn principle is often advocated as an efficient approach for variable importance, targeting well-defined theoretical quantities \citep{RFuncertainty, candes2018panning, lei2018distribution, Williamson, hooker2019please}. More precisely, the forest is retrained without a given input variable, and the decrease of accuracy with respect to the initial forest with all variables, provides the importance value. In the case of regression problems, this measure gives the proportion of explained output variance lost when a variable is removed, and has a well-grounded theoretical definition as the total Sobol index \citep{Sobol_indices}. Formally, Sobol sensitivity indices quantify the variance of output means conditional on input variables, and were recently extended to output distributions using the MMD \citep{daveiga2016new, daveiga2021kernelbased}.
Thus, we build on the drop and relearn principle and the generalized MMD-sensitivity indices, to propose a variable importance measure for DRF, defined as the MMD distance between the conditional output distribution given all inputs, and given all but one input. Such importance measure therefore quantifies how the output conditional distribution changes when a variable is removed, and can be estimated by refitting DRF removing variables one by one. 

The definition of a relevant importance measure hinges on the ultimate practical objective, which is typically categorized into two groups: (1) finding a small number of variables with a maximized accuracy, or (2) detecting and ranking all influential variables to focus on for further exploration \citep{Genuer}. 
These two goals differ when variables are dependent. For example, if two variables are highly correlated together and with the output, one of the two inputs can be removed without hurting accuracy for objective (1), since both variables convey the same information. However, both should be included for objective (2), since these two variables may have different meanings in practice for domain experts. 
In this article, we focus on objective (1), since we use the drop and relearn principle, which is only adapted in this case. For objective (2) other strategies can be used, such as Shapley effects \citep{owen2014sobol, lundberg2017unified, SHAFF}. 
The dependence between variables not only play a role in the definition of importance measures, but also present significant challenges when it comes to designing efficient algorithms for estimating it.
In the case of MMD-sensitivity indices, \citet{daveiga2021kernelbased} essentially introduces estimates adapted for the field of computer experiments, where inputs are independent, or assuming the input distribution is known, or using $k$-nearest neighbors, which struggle in non-trivial input dimensions. Using DRF, we can tackle realistic settings with dependent inputs, higher dimensions, and when only a data sample is available.

\paragraph{Motivating example.}
We consider the following scenario, combining two examples given in \citet{athey2019generalized}, where a Gaussian output $Y$ depends on two uniform variables $X^{(1)}$ and $X^{(2)}$, respectively through a shift in mean and a shift in variance, defined by
\begin{align} \label{Scenario1}
    Y \sim \mathcal{N}(0.8\cdot \1(X^{(1)} > 0), (1+\1(X^{(2)}>0))^2).
\end{align}
In addition, $\Xbf$ contains $X^{(3)}$ that is correlated with $X^{(1)}$, but does not influence $Y$, and also seven independent uniform variables. Previous variable importance measures for regression problems are designed to quantify the effect of $\Xbf$ on the conditional mean of $Y$. As such they cannot detect the influence of $X^{(2)}$ on the output distribution, as opposed to the measure introduced in this article. In particular, it correctly quantifies that the effect of $\Xbf$ on $Y$ is divided between $X^{(1)}$ and $X^{(2)}$, as seen in Table \ref{Scenario1table}. 
This is of critical importance, if the goal is to predict the distribution of $Y$ itself, or if one is generally interested in more targets than conditional expectation, such as quantile estimates. Moreover, the importance values of $X^{(3)}, \hdots, X^{(10)}$ are negligible, showing that the importance measure correctly identifies the irrelevant variables, despite the correlation of $X^{(3)}$ with $X^{(1)}$.
\begin{table}
\caption{Variable importance for data distribution defined in Equation \eqref{Scenario1}.}
\setlength{\tabcolsep}{8pt}
\centering
\begin{tabular}[t]{c c c c}
\hline
$X^{(1)}$ & $X^{(2)}$ & $X^{(3)}$ & $X^{(4)}, \hdots, X^{(10)}$\\
\hline
$0.21$ & $0.76$ & $0.007$ & $< 0.006$ \\
\hline
\end{tabular}
\label{Scenario1table}
\end{table}

\paragraph{Contributions.}
We formally define the new variable importance measure for DRF inspired by \citet{daveiga2021kernelbased}, in Section \ref{sec:vimp}, and show how it can be easily estimated by refitting DRF with variables removed one by one. We show that this estimator is consistent for the MMD-based sensitivity index in Section \ref{sec:theory}. To reduce the computational complexity, we also discuss an estimator based on the Projected Distributional Random Forest, extending the Sobol-MDA algorithm from \citet{Sobol_MDA}, and show that this also leads to a consistent estimator. Our approach can thus be seen as a natural extension of the Sobol-MDA developed for standard Random Forest to DRF. In particular, this allows for a principled variable importance measure for a multivariate dependent variable $\mathbf{Y}$. While the theoretical definition of the above importance measure is close to existing proposals \citep{daveiga2021kernelbased}, the core of our contribution is the introduction of an estimate of this importance measure using DRF, which is efficient when input variables are dependent, the output is multi-dimensional, and only a data sample is available, with unknown data distributions. 
Finally in Section \ref{sec:xp}, we analyze a broad range of simulated and real examples, showing the versatility of the new method. The examples raise from one or low-dimensional dependent variables $\Ybf$ to functional dependent data. In particular, we show the efficiency of our importance measure for recursive feature elimination on real datasets.

\section{DRF VARIABLE IMPORTANCE} \label{sec:vimp}

We first need to introduce several notations and concepts to formalize our variable importance measure for DRF. Throughout, we assume an underlying probability space $(\Omega, \mathcal{A}, \P)$, and denote by $\left(\H, \langle\cdot,\cdot\rangle_{\H}\right)$ the reproducing kernel Hilbert space (RKHS) induced by the positive definite, bounded, and continuous kernel $k\colon\R^d\times\R^d \to \R$ \citep[Chapter 2.7]{hilbertspacebook}, with dimension $d \in \mathbb{N}^{\star}$.
The kernel embedding function $\Phi$ maps any probability measure $\Q$ on $\R^d$ to an element $\Phi(\Q) \in \H$, defined by
$\Phi(\Q) = \E[k(\Zbf,\cdot)]$, with $\Zbf \sim \Q$, which is well-defined by continuity and boundedness of $k$. 
For two probability measures $\Q_1$ and $\Q_2$ on $\R^d$, the well-known Maximum Mean Discrepancy (MMD) distance \citep{gretton2012kernel} is given by 
\begin{align*}
    \MMD(\Q_1, \Q_2) = \| \Phi( \Q_1) - \Phi( \Q_2)  \|_{\H}.
\end{align*}

If the kernel $k$ is characteristic, then $\Phi$ is injective, and the MMD is a distance between probability measures.
Next, we consider an output vector of interest $\smash{\Ybf=(Y^{(1)}, Y^{(2)}, \ldots, Y^{(d)})^T \in \mathbb{R}^{d}}$, and an input vector $\smash{\Xbf=(X^{(1)}, X^{(2)}, \ldots, X^{(p)})^T \in \mathbb{R}^{p}}$ of dimension $p \in \mathbb{N}^{\star}$.
Finally, we focus on $\mu(\xbf)$, the Hilbert space embedding of the multivariate conditional distribution $\PYgx$ for a given input point $\xbf \in \mathbb{R}^{p}$, i.e.,
\begin{align*}
    \mu(\xbf) \defeq \Phi(\PYgx) = \E[k(\Ybf, \cdot) \mid \Xbf = \xbf] \in \H.    
\end{align*}

\paragraph{Theoretical importance measure.}
Now, in the same spirit as \cite{daveiga2021kernelbased}, we show how to obtain a variable importance measure based on the MMD embedding of the conditional distribution and the drop and relearn principle.
We first consider the estimation of the conditional mean $\smash{\tau(\xbf) \defeq \E[Y \mid \Xbf=\xbf]}$, for $d=1$. In this case, one may define the total Sobol index \citep{Sobol_indices} as
\begin{align}\label{oldmeasure}
    \mathrm{ST}^{(j)} \defeq \frac{\E[\Var[\tau(\Xbf) \mid \Xbf^{(-j)} ]]}{\Var[Y]},
\end{align}
where $\Xbf^{(-j)}=\left( X^{(\ell)}\right)_{\ell \neq j}$, and $\Var[\cdot]$ is the variance. As mentioned above, this is the expected reduction in output explained variance, once the $j$-th input variable is removed. Another way to see this measure is to write the numerator of $\mathrm{ST}^{(j)}$ as
\begin{align*}
    \E[\Var[\tau(\Xbf) \mid \Xbf^{(-j)}]] = \E[d_E(\tau(\Xbf), \E[ \tau(\Xbf) \mid \Xbf^{(-j)} ])^2],
\end{align*}
where $d_E$ is the Euclidean distance. That is, we consider the distance $d_E$ between the estimates conditional on respectively $\Xbf$ and $\Xbf^{(-j)}$.
For other target quantity than $\tau$, we need another relevant distance. For $\mu(\xbf) \in \H$, a natural choice is the distance induced by the norm $\|\cdot\|_{\H}$, i.e. $d(\xi, \xi') = \|\xi - \xi'\|_{\H}$ for $\xi, \xi' \in \H$, which leads to the variance operator $\VH$ in $\H$, defined by $\VH[\xi \mid \Xbf] = \E[ \| \xi - \E[ \xi \mid \Xbf] \|_{\H}^2 \mid \Xbf]$.
We can now formalize our theoretical importance measure as the generalized total Sobol index, defined by
\begin{align}\label{DRFvarimportance}
    \VI^{(j)} \defeq \frac{\E[ \VH[\mu(\Xbf) \mid \Xbf^{(-j)}]]}{\VH[\mu(\Xbf)]},
\end{align}
which can also be written using the MMD, as stated in the following proposition: $\VI^{(j)}$ quantifies the distance between the conditional distribution $\PYgX$ with all input variables involved and $\PYgXmj$ when one variable is dropped, with respect to the distance between $\PY$ and $\PYgX$. Therefore, this measure has different goals than the traditional Sobol indices. While the latter is designed to detect changes in the conditional expectation of the response variable, $\VI^{(j)}$ is designed to detect any change in the distribution of $\Ybf$, as we will formally prove in the following section.
All proofs of propositions and theorems are gathered in Appendix~B.

\begin{restatable}{proposition}{VIMMD}\label{prop:VI_MMD}
    If $\VI^{(j)}$ is the generalized total Sobol index defined by Equation (\ref{DRFvarimportance}), then we have
    \begin{align*}
        \VI^{(j)} &= \frac{\E[\MMD^2(\PYgX, \PYgXmj)]}{\E[\MMD^2(\PY, \PYgX)]}, \\
        \VI^{(j)} &= 1 - \frac{\E[\MMD^2(\PY, \PYgXmj)]}{\E[\MMD^2(\PY, \PYgX)]}.
    \end{align*}
\end{restatable}

Importantly, our importance measure is defined with a different normalization constant than in \citet{daveiga2021kernelbased}, since we use $\VH[\mu(\Xbf)]$ instead of $\VH[k(\Ybf,\cdot)] = \E[k(\Ybf, \Ybf)] - \E[k(\Ybf, \Ybf')]$, where $\Ybf'$ is and independent copy of $\Ybf$. Indeed, \citet{daveiga2021kernelbased} introduces MMD-based sensitivity indices in the specific settings of computer experiments, where outputs are deterministic functions of inputs. In this case, the two normalization constants coincide, since $\mu(\Xbf) = k(\Ybf,\cdot)$. On the other hand, we consider distributions where the variability of $\Ybf$ is only partially explained by $\Xbf$, with potentially an explained variability $\VH[\mu(\Xbf)]$ much smaller than the total output variations $\VH[k(\Ybf,\cdot)]$. Therefore, using this last quantity as normalization constant would often lead to small importance values on real data, and we rather define $\VI^{(j)}$ with respect to the variability $\VH[\mu(\Xbf)]$ explained by all inputs. Notice that, if $\Ybf$ and $\Xbf$ are independent, $\VH[\mu(\Xbf)] = 0$, and no input is influential by construction.

\paragraph{Variable importance estimate.}
We assume that an independent and identically distributed data sample $\Zcal_n = \{\Zbf_i\}_{i=1}^n$ of size $n$ is available, where the $i$-th observation is defined by $\Zbf_i=(\Xbf_i, k(\Ybf_i, \cdot)) \in \R^p \times \H$. 
DRF provides nonparametric estimates of the distribution 
of the multivariate response $\Ybf$, conditional on the potentially high-dimensional input vector $\Xbf$. That is, for a given query point $\xbf \in \mathbb{R}^{p}$, DRF estimates the Hilbert space embedding $\mu(\xbf)$ of $\PYgx$, denoted by $\mu_{N,n}(\xbf)$, and defined as the average of the $N$ tree estimates. Formally, $\mu_{N,n}(\xbf)$ writes
\begin{equation*}\label{rewriting}
    \mu_{N,n}(\xbf) = \frac{1}{N} \sum_{\ell=1}^N T_n(\mathbf{x}; \varepsilon_{\ell}, \Zcal_{\ell}),
\end{equation*}
where $\Zcal_{\ell} = \{\Zbf_{\ell_1}, \ldots, \Zbf_{\ell_{s_n}}\}$ is a random subset of $\Zcal_n$ of size $s_n$ chosen for constructing the $\ell$-th tree, $\varepsilon_\ell$ is a random variable capturing the randomness in growing the $\ell$-th tree such as the choice of the splitting candidates, and $T_n(\mathbf{x}; \varepsilon_\ell, \Zcal_\ell)$ denotes the output of a single tree. More precisely, the tree estimate at query point $\xbf$, constructed from $\varepsilon_{\ell}$ and $\Zcal_{\ell}$, is given by the average of the terms $k(\Ybf_i, \cdot)$ over all data points $\mathbf{X}_i$ contained in the leaf $\Lcal_{\ell}(\mathbf{x})$ where $\xbf$ falls, i.e.,
\begin{align*}
    T_n(\xbf; \varepsilon_{\ell}, \Zcal_{\ell})=\sum_{i=1}^{s_n} \frac{\1(\Xbf_{\ell_i} \in \mathcal{L}_{\ell}(\xbf))}{|\mathcal{L}_{\ell}(\xbf)|} k(\Ybf_{\ell_i}, \cdot).
\end{align*}
Then, we can express the DRF output $\mu_{N,n}(\xbf)$ as an adaptive nearest neighbor estimate, defined by 
\begin{align*}
    \mu_{N,n}(\xbf) = \frac{1}{n} \sum_{i=1}^n w_i(\xbf) k(\Ybf_i, \cdot),
\end{align*}
where the weight for each training observation $\Xbf_i$ writes $w_i(\xbf) = 1/N \sum_{\ell=1}^N \1(\Xbf_{i} \in \mathcal{L}_{\ell}(\xbf))/|\mathcal{L}_{\ell}(\xbf)|$.

Finally, we build an estimate $\VI_n^{(j)}$ of $\VI^{(j)}$ using an initial DRF estimate $\mu_{N,n}(\xbf)$ fit with all inputs involved, combined with the DRF estimate retrained with the $j$-th variable removed, denoted by $\mu_{N,n}(\xbf^{(-j)})$. Thus, using an independent sample $\Xbf'_1, \ldots \Xbf'_n$, we define
\begin{align}\label{Injdef}
    \VI_n^{(j)} = \frac{ \sum_{i=1}^n \norm{\mu_{N,n}(\Xbf'_i) - \mu_{N,n}(\Xbf_i'^{(-j)}) }_{\H}^2 }{\sum_{i=1}^n \norm{ \mu_{N,n}(\Xbf'_i) - \overline{\mu_{N,n}} }_{\H}^2} - \VI_n^{(0)},
\end{align}
where $\overline{\mu_{N,n}} = \sum_{i=1}^n \mu_{N,n}(\Xbf'_i)/n$, and $\VI_n^{(0)}$ is defined as the first term of $\smash{\VI_n^{(j)}}$, but with the DRF $\mu'_{N,n}(\Xbf'_i)$, retrained with still all inputs involved but new independent randomizations of the trees $\varepsilon_1', \hdots, \varepsilon_N'$, i.e. 
\begin{align*}
    \VI_n^{(0)} = \frac{\sum_{i=1}^n \norm{\mu_{N,n}(\Xbf'_i) - \mu'_{N,n}(\Xbf'_i)}_{\H}^2 }{\sum_{i=1}^n \norm{ \mu_{N,n}(\Xbf'_i) - \overline{\mu_{N,n}}}_{\H}^2}.
\end{align*}
This is used in Equation (\ref{Injdef}) to mitigate the finite sample bias.
In practice, $\smash{\VI_n^{(j)}}$ is simply computed through vector and matrix multiplications. To state the formula, we introduce the kernel matrix $ \smash{\Kbf=(k(\Yibf,\Ybf_{j} ))_{i,j \in \{1,\ldots n\}}}$, the DRF weight vectors $\wbf(\xbf) = (w_1(\xbf), \hdots, w_n(\xbf))$, and $\smash{\wbf(\xbf^{(-j)}) = (w_1(\xbf^{(-j)}), \hdots, w_n(\xbf^{(-j)}))}$ for the retrained DRF without the $j$-th variable. Moreover, we consider the vector $ \smash{\kbf=( k(\Ybf_1,\cdot ) , \ldots, k(\Ybf_{n},\cdot ) )^{\top}}$, and the mean weight over the independent sample $\smash{\bar{\wbf} = \sum_{i=1}^n \wbf(\Xbf'_i)/n}$. Then, the forest estimates writes $\mu_{N,n}(\xbf) = \wbf(\xbf)^{\top} \kbf$ and $\mu_{N,n}(\xbf^{(-j)}) = \wbf(\xbf^{(-j)})^{\top} \kbf$, and $\smash{\VI_n^{(j)}}$ is calculated with the following formula,
\begin{align} \label{eq:VIn}
    \VI_n^{(j)} = \Big\{ \sum_{i=1}^n \big[ \wbf(\Xbf'_{i}) - \wbf&(\Xbf_{i}'^{(-j)}) \big]^{\top} \Kbf \\[-1.2em] &\times \big[ \wbf(\Xbf'_{i}) - \wbf(\Xbf_{i}'^{(-j)}) \big] \Big\} \nonumber \\[-0.8em]
    \times \Big\{ \sum_{i=1}^n \big[ \wbf(\Xbf'_{i}) - &\bar{\wbf} \big]^{\top} \Kbf \big[ \wbf(\Xbf'_{i}) - \bar{\wbf} \big] \Big\}^{-1} \hspace*{-1mm} - \VI_n^{(0)}, \nonumber
\end{align}
where $\VI_n^{(0)}$ takes the same form as the first term.
Thus, we in fact consider the difference in weights, with each element weighted by the kernel matrix $\Kbf$. In turn, this is standardized by the estimated variance of the embedding of $\Ybf\mid \Xbf$. For the sake of clarity, $\smash{\VI_n^{(j)}}$ is formalized with an independent dataset, but out-of-bag predictions can also be used instead.
We will see in the next section that this variable importance algorithm for DRF is consistent with respect to the theoretical importance measure defined in Equation (\ref{DRFvarimportance}), and thus provides an efficient assessment of the impact of each variable on the output conditional distribution.

\section{THEORETICAL PROPERTIES} \label{sec:theory}

The construction of our variable importance algorithm is based on well-defined quantities from Equations (\ref{oldmeasure}) and (\ref{DRFvarimportance}) and Proposition \ref{prop:VI_MMD}, and therefore enjoys good theoretical properties as we show throughout this section. To state our results, we need to formalize several assumptions, and we first characterize the required kernel properties.
\begin{enumerate}[label=(\textbf{K\arabic*})]
    \item\label{kernelass1} The kernel $k$ is bounded, and the function $(\xbf,\ybf) \mapsto k(\xbf,\ybf)$ is (jointly) continuous.
    \item\label{kernelass3} The kernel $k$ is characteristic.
\end{enumerate}
In particular, Assumption \ref{kernelass3} implies that the kernel embedding function $\Phi$ is injective, and then, for two probability measures $\Q_1$ and $\Q_2$, $\| \Phi(\Q_1) - \Phi(\Q_2) \|_{\H} = 0$ implies $\Q_1 = \Q_2$, as explained in \citet{optimalestimationofprobabilitymeasures, simon2020metrizing}. For example, all these assumptions are met for the Gaussian kernel, which is the standard kernel in DRF, see e.g., \citet[Appendix A]{DRF-paper}. When the chosen kernel $k$ satisfies these assumptions, our importance measure defined in Equation (\ref{DRFvarimportance}) detects any change in the output conditional distribution when a variable is removed, as stated in the following proposition.

\begin{restatable}{proposition}{VIpositive}\label{prop:VI_positive}
  Assume that Assumptions~\ref{kernelass1}-\ref{kernelass3} holds and that for each $\xbf$ in a set with nonzero probability, $\PYgx \neq \PYgxmj$, then $0 < \VI^{(j)} \leq 1$, and otherwise $\VI^{(j)} = 0$.
\end{restatable}

To deepen our discussion, we also need assumptions about the data distribution, given below.
\begin{enumerate}[label=(\textbf{D\arabic*})]
    \item\label{dataass1} The observations $\mathbf{X}_1,\ldots, \mathbf{X}_n$ are independent and identically distributed on $[0,1]^p$, with a density bounded from below and above by strictly positive constants.
        \item\label{dataass2} The mapping $\xbf \mapsto \mu(\xbf)=\E[ k(\Ybf,\cdot)  \mid \Xbf \myeq \xbf] \in \H $ is Lipschitz.
\end{enumerate}
Assumption \dataass{1} is standard when analyzing Random Forest, see e.g., \citet{wager2017estimation, athey2019generalized, DRF-paper, näf2023confidence}. Assumption \dataass{2} can be restated as
\begin{align*}
    \MMD(\P_{\Ybf \mid \Xbf=\xbf_1}, \P_{\Ybf \mid \Xbf=\xbf_2}) \leq L \| \xbf_1 - \xbf_2\|_{\R^p},
\end{align*}
for all $\xbf_1, \xbf_2 \in \R^p$, and some Lipschitz constant $L > 0$. Thus, whenever $\xbf_1$ and $\xbf_2$ are close, the corresponding distributions $\P_{\Ybf \mid \Xbf=\xbf_1}$ and $\P_{\Ybf \mid \Xbf=\xbf_2}$ need to be close in terms of MMD distance. This corresponds to a natural generalization of the Lipschitz condition usually assumed for standard Random Forest \citep{wager2018estimation, athey2019generalized}.

\paragraph{Consistency of DRF importance.}
To develop our theory, we consider infinite forests, a standard simplification also employed by~\citet{scornet2015consistency, wager2018estimation, athey2019generalized}, where the considered forest is defined as the limit when the number of trees $N$ grows to infinity. Such forest estimator $\hmun$ is obtained by averaging all $\binom{n}{s_n}$ possible subsets of $\{\Zbf_{i}\}_{i=1}^n$ of size $s_n$, and taking the expectation over the tree randomization $\varepsilon$.
This idealized version of our DRF predictor, which we will denote by $\hmun$ from now onwards, is given by 
\begin{equation}\label{finalestimator}
    \hmun = \binom{n}{s_n}^{-1} \hspace*{-5mm} \sum_{i_1 < \cdots < i_{s_n}} \hspace*{-4mm} \E\left[T_n(\xbf; \varepsilon, \{\mathbf{Z}_{i_1}, \ldots, \mathbf{Z}_{i_{s_n}}\}) \mid \Zcal_n \right].
\end{equation}

Following~\citet{wager2018estimation, athey2019generalized, DRF-paper, näf2023confidence}, the forest construction enforces that trees are honest, symmetric, $\alpha$-regular, each node may split on all variables with a positive probability, and the subsample size $s_n$ is defined by $s_n = n^{\beta}$, with $0 < \beta < 1$. Theses characteristics are formalized in Assumptions \forestass{1}--\forestass{5} in Appendix B. We now prove the consistency of the proposed DRF variable importance algorithm, stated in Equation \eqref{Injdef}. First, we slightly strengthen the result of consistency in \citet[Theorem 1]{DRF-paper}.

\begin{restatable}{proposition}{meanconsistency}\label{thm: meanconsistency}
Assume that the forest construction satisfies the properties~\textbf{(F1)}-\textbf{(F5)}. Additionally, assume  that $k$ meets Assumption~\ref{kernelass1}, and that \ref{dataass1} and \ref{dataass2} hold. Then, we have consistency of $\hmun$ in~\eqref{finalestimator} with respect to the RKHS norm in mean, that is
\begin{align*}
\E[\norm{\hmunX- \mu(\Xbf)}_\H] \longrightarrow 0.
\end{align*}
\end{restatable}

The consistency of $\hmunbruteX$ directly follows from Proposition \ref{thm: meanconsistency}, since $\hmunbruteX$ is trained on the data with the $j$-th input variable removed. We only need the following additional Lipschitz assumption to satisfy Assumption~\ref{dataass2} with a reduced set of inputs.
\begin{enumerate}[label=(\textbf{D\arabic*})]
    \setcounter{enumi}{2}
    \item\label{dataass3} The mapping $\xbf^{(-j)} \mapsto \E[ k(\Ybf,\cdot)  \mid \Xbf^{(-j)} \myeq \xbf^{(-j)}] \in \H $ is Lipschitz.
\end{enumerate}
\begin{restatable}{proposition}{projmeanconsistency}\label{thm: proj_meanconsistency}
Under the assumptions of Proposition \ref{thm: meanconsistency}, and provided that Assumption \ref{dataass3} is satisfied, we have
$\E[\norm{\hmunbruteX- \mu(\Xbf^{(-j)})}_\H] \longrightarrow 0$.
\end{restatable}

We deduce the consistency of our variable importance algorithm from the last two propositions.
\begin{restatable}{theorem}{VIconsistency}\label{thm: VIconsistency}
Assume that the forest construction satisfies the properties~\textbf{(F1)}-\textbf{(F5)}. Additionally, assume  that $k$ meets Assumption~\ref{kernelass1}, and that \ref{dataass1}-\ref{dataass3} hold. Then, we have consistency of $\smash{\VI_n^{(j)}}$ in \eqref{Injdef}, that is
\begin{equation*}
\VI_n^{(j)} \stackrel{p}{\longrightarrow} \VI^{(j)}.
\end{equation*}
\end{restatable}

\paragraph{Projected DRF.}
The computation of our importance measure for all input variables involves a DRF retrain for each variable, i.e. $p+1$ training. While this approach is efficient for moderate dimensions, it is not tractable in high-dimensional settings. Indeed, \citet{DRF-paper} state that the computational complexity of a single DRF fit is $\O(B \times N \times p \times n \log n)$, where $B$ is the number of random features to approximate the MMD statistics for the splitting criterion. Therefore, our importance algorithm has a quadratic complexity with respect to the input dimension $p$. In fact, this limitation also occurs for Breiman's forest in the case of regression problems. A solution was proposed by \citet{Sobol_MDA} with the Projected Random Forest, whose complexity is independent of $p$, once the initial forest with all inputs is trained, and is therefore strongly efficient in high dimensional settings. 
The principle of projected forests is to eliminate a given input variable from the prediction mechanism of the forest by the projection of the tree partitions on the subspace generated by all other variables. In practice, the associated predictions can be easily computed using the original trees, and simply ignoring splits involving the discarded variable by sending data points on both sides of such splits. Following this procedure, both training points and the new query point fall in multiple terminal leaves of the original partition. Then, cells are intersected to get the training points falling in the exact same collections of leaves as the considered query point, and finally used to compute the projected tree prediction. The projected forest was later extended to Shapley effects in \citet{SHAFF}, and to local importance measures by \citet{amoukou2022consistent}.

As DRF is essentially a Random Forest with the dependent variable taking values in the RKHS $\H$, the same approach can be adapted as well, to obtain a Projected Distributional Random Forest. Applying the arguments of \citet{Sobol_MDA}, the projection approach reduces the computational complexity from the $\O(B\times N \times p^2 \times n \log(n)^2)$ of a DRF fit for each variable to $\O(B \times N \times n \log(n)^3)$. Here, we show that this approach still leads to a consistent estimator of $\VI^{(j)}$. In the sequel, we denote by $\smash{\hmunprojX}$ the projected DRF estimate, indicating that the projection was done after fitting DRF on the full data. This is in contrast to $\hmunbruteX$ which was already fitted with $X^{(j)}$ removed. 

\begin{restatable}{proposition}{mujconsistency}\label{thm: meanconsistencyprojected}
Assume that the forest construction satisfies the properties~\textbf{(F1)}-\textbf{(F5)}. Additionally, assume  that $k$ meets Assumption~\ref{kernelass1}, and that \ref{dataass1}-\ref{dataass3} hold. Then, we have consistency of $\smash{\hmunprojx}$ in probability,
\begin{align*}
  \norm{\hmunprojx- \mu(\xbf^{(-j)})}_\H = \O_{p}\left(n^{-\gamma} \right),
\end{align*}
for any $\gamma \leq \frac{1}{2} \min\left( 1- \beta, \frac{\pi \log(1-\alpha)}{p \log(\alpha)} \cdot \beta \right)$, where $\alpha$ and $\beta$ are chosen in \forestass{4} and \forestass{5} respectively. Moreover,
\begin{equation*}
\E[\norm{\hmunprojX- \mu(\Xbf^{(-j)})}_\H] \longrightarrow 0.
\end{equation*}
\end{restatable}

We then redefine the variable importance by simply exchanging $\hmunbruteXl$ by $\hmunprojXl$, and obtain
\begin{align}\label{Injdefprojected}
    \VI_{n, \textrm{proj}}^{(j)} = \frac{ \sum_{i=1}^n \norm{\mu_n(\Xbf'_i) - \hmunprojXl }_{\H}^2 }{\sum_{l=1}^n \norm{ \mu_n(\Xbf'_i) - \frac{1}{n} \sum_{i=1}^n \mu_n(\Xbf'_i) }_{\H}^2},
\end{align}
which is also consistent, as stated in this last result.
\begin{restatable}{theorem}{VIprojconsistency}\label{thm: VIprojconsistency}
Assume that the forest construction satisfies the properties~\textbf{(F1)}-\textbf{(F5)}. Additionally, assume  that $k$ meets Assumption~\ref{kernelass1}, and that~\ref{dataass1}-\ref{dataass3} hold. Then, $\smash{\VI_{n, \textrm{proj}}^{(j)}}$ in \eqref{Injdefprojected} is consistent, that is
\begin{equation*} 
 \VI_{n, \textrm{proj}}^{(j)} \stackrel{p}{\to} \VI^{(j)}.
\end{equation*}
\end{restatable}

In addition, we note that the last step of our algorithm is to compute $\smash{\VI_n^{(j)}}$ through Equation (\ref{eq:VIn}), which involves a matrix multiplication of complexity $\smash{\O(n^2)}$. Following \citet{DRF-paper}, we can use the Random Fourier approximation to compute the MMD of Equation (\ref{eq:VIn}), with a linear complexity with respect to the sample size $n$. Alternatively, once all predictions are computed, it is also possible to subsample the data to compute Equation (\ref{eq:VIn}), with a size of typically $1000$ points for large samples. This leads to a $O(1)$ complexity, and a minor impact on the accuracy of $\smash{\VI_n^{(j)}}$, since Equation (\ref{eq:VIn}) is simply an average over a large number of points.

\section{EXPERIMENTS} \label{sec:xp}

We run several batches of experiments to show the high performance of our importance measure on both simulated and real data, especially with respect to the main competitors. In particular, we run comparisons with the native DRF importance \citep{drf-package}, based on split frequencies and denoted by vimp-drf.
For univariate output cases, we also add the MDA \citep{breiman2001random} and Sobol-MDA \citep{Sobol_MDA} algorithms.
While Sobol-MDA also focuses on objective (1), MDA and vimp-drf are not specifically tailored for objective (1) or (2). Nonetheless, vimp-drf appears to be the only existing competing variable importance measure for DRF, to our best knowledge. Therefore, we only show that our method performs better than vimp-drf for our objective of interest. Additionally, the MDA is the most widely used importance measure for Breiman's forest, and is thus a useful baseline.
Finally, we use the Gaussian kernel with the median heuristic, and a number of random features $B = 10$, which is standard in the DRF implementation \citep{drf-package, DRF-paper}.
Code to reproduce the experiments is available in the Supplementary Material.

\subsection{Simulated data}

\paragraph{Univariate output.}
We run a first experiment with  a univariate output to compare our proposed variable importance algorithm to the existing vimp-drf based on split frequencies, and standard methods for regression forests. 
Hence, we consider a Gaussian input vector of dimension $p = 10$, where all variables have unit variance, and each pair of distinct variables have a correlation of $0.5$, except that $\mathrm{Cov}(X^{(1)}, X^{(10)}) = 0.9$. Then, the output is defined as 
\begin{align*}
    Y \sim \mathcal{N}(2X^{(1)} + X^{(2)}, (2|X^{(3)}| + 2|X^{(4)}| + 2|X^{(5)}|)^2).
\end{align*}
Based on this data distribution, we run the following experiment: a data sample of size $n = 3000$ is drawn, a DRF is fit with $N=500$ trees, and both $\smash{\VI_n^{(j)}}$ and vimp-drf are computed. We also fit a regression forest and compute the MDA \citep{breiman2001random} and Sobol-MDA \citep{Sobol_MDA}. This procedure is repeated $10$ times for uncertainties, and the average importance values are reported in Table \ref{table:xp_univariate}. Most standard deviations are small, and displayed in Table $1$ in Appendix A. Clearly, $\smash{\VI_n^{(j)}}$ is the only algorithm to identify the five relevant variables as the most important ones. On the other hand, both vimp-drf and MDA rank $X^{(10)}$ in second position, because of its strong correlation with $X^{(1)}$, although $X^{(10)}$ is not involved in the distribution of $Y$. Since variables $X^{(3)}$, $X^{(4)}$, and $X^{(5)}$ are not involved in the mean of $Y$ but only in its variance, they are only identified as important by $\smash{\VI_n^{(j)}}$ and vimp-drf, but not by the MDA and Sobol-MDA, as expected. While the Sobol-MDA gives a negligible importance to all variables not involved in the mean of $Y$, the MDA gives high negative values to the relevant variables $X^{(3)}$, $X^{(4)}$, and $X^{(5)}$. This phenomenon is not really surprising given the MDA's flaws, mentioned in the introduction, and extensively discussed in the literature. Finally, the regression forest has an explained variance of $13\%$, because of the strong noise involved in the definition of $Y$. This explains the quite small values of $X^{(1)}$, $X^{(2)}$ given by the Sobol-MDA, which estimates the proportion of output variance lost when a given input variable is removed.
\begin{table}
\setlength{\tabcolsep}{1pt}
\centering
\begin{tabular}{c c}
  \hline
   $X^{(j)}$ & $\VI_n^{(j)}$ \\ 
  \hline
  \blueft{$X^{(1)}$} & 0.181 \\ 
  \blueft{$X^{(4)}$} & 0.073 \\ 
  \blueft{$X^{(5)}$} & 0.073 \\ 
  \blueft{$X^{(2)}$} & 0.072 \\ 
  \blueft{$X^{(3)}$} & 0.065 \\ 
  $X^{(10)}$ & 0.010 \\ 
  $X^{(7)}$ & 0.005 \\ 
  $X^{(6)}$ & 0.005 \\ 
  $X^{(9)}$ & 0.005 \\ 
  $X^{(8)}$ & 0.005 \\ 
   \hline
\end{tabular}
\begin{tabular}{c c}
  \hline
   $X^{(j)}$ & v-drf \\ 
  \hline
  \blueft{$X^{(1)}$} & 0.649 \\ 
  $X^{(10)}$ & 0.096 \\ 
  \blueft{$X^{(2)}$} & 0.062 \\ 
  \blueft{$X^{(5)}$} & 0.059 \\ 
  \blueft{$X^{(4)}$} & 0.056 \\ 
  \blueft{$X^{(3)}$} & 0.050 \\ 
  $X^{(6)}$ & 0.007 \\ 
  $X^{(9)}$ & 0.007 \\ 
  $X^{(7)}$ & 0.007 \\ 
  $X^{(8)}$ & 0.007 \\ 
   \hline
\end{tabular}
\begin{tabular}{c c}
  \hline
   $X^{(j)}$ & MDA \\ 
  \hline
  \blueft{$X^{(1)}$} & 6.47 \\ 
  $X^{(10)}$ & 2.56 \\ 
  \blueft{$X^{(2)}$} & 1.33 \\ 
  $X^{(6)}$ & 0.25 \\ 
  $X^{(8)}$ & 0.24 \\ 
  $X^{(7)}$ & 0.18 \\ 
  $X^{(9)}$ & 0.18 \\ 
  \blueft{$X^{(3)}$} & -0.52 \\ 
  \blueft{$X^{(5)}$} & -0.62 \\ 
  \blueft{$X^{(4)}$} & -0.90 \\ 
   \hline
\end{tabular}
\begin{tabular}{c c}
  \hline
   $X^{(j)}$ & S-MDA \\ 
  \hline
  \blueft{$X^{(1)}$} & 0.014 \\ 
  \blueft{$X^{(2)}$} & 0.012 \\ 
  $X^{(8)}$ & -0.001 \\ 
  $X^{(7)}$ & -0.003 \\ 
  $X^{(9)}$ & -0.003 \\ 
  $X^{(6)}$ & -0.003 \\ 
  \blueft{$X^{(5)}$} & -0.003 \\ 
  \blueft{$X^{(3)}$} & -0.004 \\ 
  \blueft{$X^{(4)}$} & -0.004 \\ 
  $X^{(10)}$ & -0.006 \\ 
   \hline
\end{tabular}
\caption{Variable importance for the univariate output experiment for $\smash{\VI_n^{(j)}}$, vimp-drf (v-drf), MDA, and Sobol-MDA (S-MDA).}
\label{table:xp_univariate}
\end{table}

\paragraph{Bivariate output.}
A main feature of DRF is to handle multivariate outputs, and we therefore focus our second simulated experiment on such a case. We consider $p = 10$ input variables, following a uniform distribution on the unit cube, and two uniform outputs defined by $Y^{(1)} \sim \mathcal{U}(X^{(1)}, 1 + X^{(1)})$ and $Y^{(2)} \sim \mathcal{U}(0, X^{(2)})$.
Next, we draw a sample of size $n = 500$, fit a DRF of $N = 500$ trees, and finally compute our importance measure, as well as vimp-drf.
Table \ref{table:xp_bivariate} displays the mean importance over $10$ repetitions for each variable, where standard deviations are small, and thus omitted. Clearly, both methods identify the relevant variables $\smash{X^{(1)}}$ and $\smash{X^{(2)}}$ as the most relevant ones. However, $\smash{\VI_n^{(j)}}$ identifies $X^{(1)}$ as more important than $X^{(2)}$, as opposed to vimp-drf. By definition of $\smash{\VI_n^{(j)}}$ and the MMD distance, if $\smash{X^{(1)}}$ is removed from the training data, the conditional distribution estimated by DRF is closer to the true target than when $\smash{X^{(2)}}$ is removed. This shows that vimp-drf based on split frequencies can be misleading. Also notice that the importance given by $\smash{\VI_n^{(j)}}$ is negligible for all irrelevant variables, while vimp-drf gives higher values.
\begin{table}
\centering
\begin{tabular}{c c c c c c c}
  \hline
 & \blueft{$X^{(1)}$} & \blueft{$X^{(2)}$} & $X^{(6)}$ & $X^{(8)}$ & $X^{(3)}$ & $X^{(10)}$ \\ 
  \hline
$\VI_n^{(j)}$ & 0.68 & 0.41 & $9.10^{-4}$ & $8.10^{-4}$ & $7.10^{-4}$ & $6.10^{-4}$ \\ 
vimp-drf & 0.19 & 0.70 & 0.01 & 0.02 & 0.01 & 0.02 \\ 
   \hline
\end{tabular}
\caption{Top six variables for the bivariate output experiment for $\VI_n^{(j)}$ and vimp-drf.}
\label{table:xp_bivariate}
\end{table}
Finally, we also take advantage of this second experiment to show the scalability of our algorithm with respect to the sample size $n$. We run the same experiment with increasing sample sizes, up to $n = 10 000$. For the last step of the procedure, where we aggregate predictions using Equation (\ref{eq:VIn}), we use a subsample of fixed size $1000$, to preserve a linear complexity---see the end of Section $3$. Figure \ref{fig_xp_2_nsamples} displays the results and shows that the importance values (again averaged over $10$ repetitions) are roughly constant as the sample size increases, with a slight increase for $X^{(1)}$ and $X^{(2)}$.
\begin{figure}
	\begin{center}
		\includegraphics[scale = 0.45]{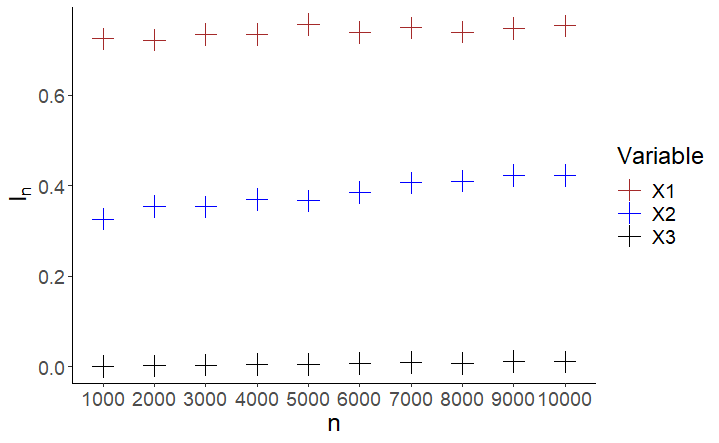}
		\caption{Values of $\VI_n^{(1)}$, $\VI_n^{(2)}$, and $\VI_n^{(3)}$, with an increasing sample size for the bivariate output experiment.}
		\label{fig_xp_2_nsamples}
	\end{center}
\end{figure}

\paragraph{High-dimensional case.}
Variable selection is frequently performed in high-dimensional settings. The goal of this third experiment is to show the good behavior of our algorithm in such cases. We consider again the experiment of the previous paragraph with a bivariate output. We simply set $p = 1000$ instead of $p = 10$, by adding uniform input variables, and keep all the other settings untouched. We obtain the results displayed in Table \ref{table:xp_bivariate_highdim}.
The good performance of our algorithm is preserved, and it still outperforms the existing competitor, which wrongly identifies variable $X^{(2)}$ as more important than $X^{(1)}$, and dilutes the importance of the two relevant variables compared to the original case with $p=10$---see Table \ref{table:xp_bivariate}.
Variables are ordered in decreasing order of $\smash{\mathrm{I}_n^{(j)}}$ in Table \ref{table:xp_bivariate_highdim}: although there are $1000$ variables involved, the highest value among the irrelevant variables is still small.
\begin{table}
\setlength{\tabcolsep}{3pt}
\centering
\begin{tabular}{c c c c c c c}
  \hline
    & \blueft{$X^{(1)}$} & \blueft{$X^{(2)}$} & $X^{(371)}$ & $X^{(698)}$ & $X^{(866)}$ & $X^{(705)}$ \\ 
  \hline
    $\mathrm{I}_n^{(j)}$ & 0.57 & 0.29 & 0.02 & 0.02 & 0.02 & 0.02 \\ 
    vimp-drf & 0.03 & 0.05 & 0.001 & 0.001 & 0.001 & 0.001 \\ 
   \hline
\end{tabular}
\caption{Top six variables for the bivariate output experiment with $p=1000$.}
\label{table:xp_bivariate_highdim}
\end{table}

\paragraph{Functional output.}
For this last simulated experiment, we address the more challenging case of a functional output.
That is, we assume that conditional on $\Xbf$, $Y$ is a Gaussian process \citep{rasmussen2006}: $Y(t) = X^{(1)} +  f(t)$, where for all $t \in \rset$, $f(t) \sim \mathcal{GP}\big(0, k_{\Xbf}(t, t)\big)$, with $k_{\Xbf}$ being a Gaussian kernel with bandwidth parameter $1/X^{(2)}$. Then, the vector $\Ybf$ given $\Xbf$ is obtained by sampling from this Gaussian process on a fixed regular grid of $[-5, 5]$ of size $d = 30$. As before, $\Xbf$ is uniformly distributed on the unit cube, with $p = 10$, and we also set $n = 2000$ and $N = 500$ trees.
Table \ref{table:xp_functional} shows that the DRF importance measures obtain the right ordering, and especially detect the dependence of $\Ybf$ on $X^{(2)}$, which is a notoriously difficult problem. 
\begin{table}
\centering
\begin{tabular}{c c c c c c c }
  \hline
 & \blueft{$X^{(1)}$} & \blueft{$X^{(2)}$} & $X^{(6)}$ & $X^{(3)}$ & $X^{(4)}$ & $X^{(9)}$ \\ 
  \hline
$\VI_n^{(j)}$ & 0.70 & 0.34 & 0.03 & 0.03 & 0.03 & 0.03 \\ 
vimp-drf & 0.61 & 0.26 & 0.02 & 0.02 & 0.02 & 0.02 \\ 
   \hline
\end{tabular}
\caption{Top six variables for the functional output experiment for $\VI_n^{(j)}$ and vimp-drf.}
\label{table:xp_functional}
\end{table}

\subsection{Recursive feature elimination for real data}

\begin{figure*}
    \begin{center}
        \includegraphics[width=0.45 \textwidth]{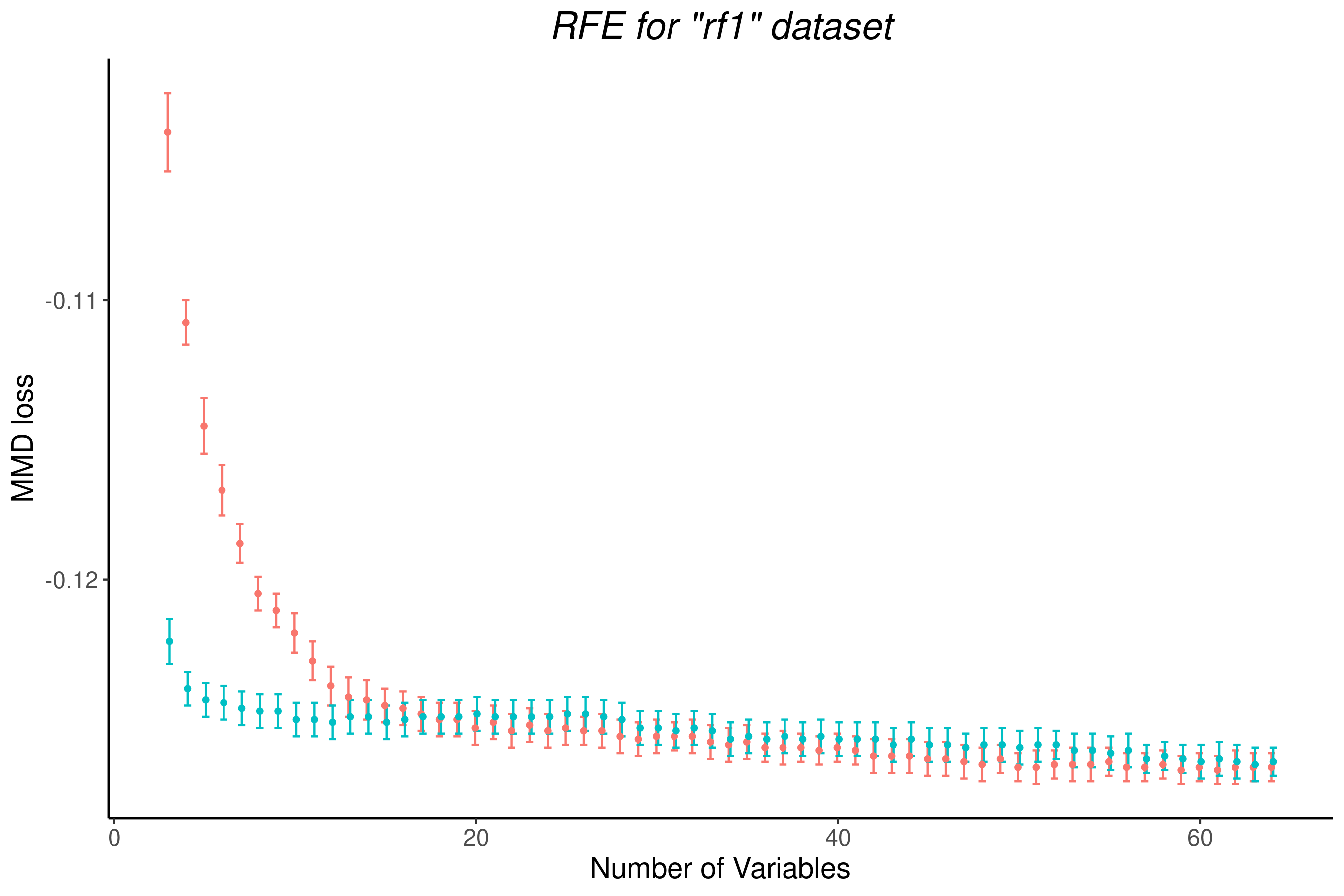} 
        \includegraphics[width=0.45 \textwidth]{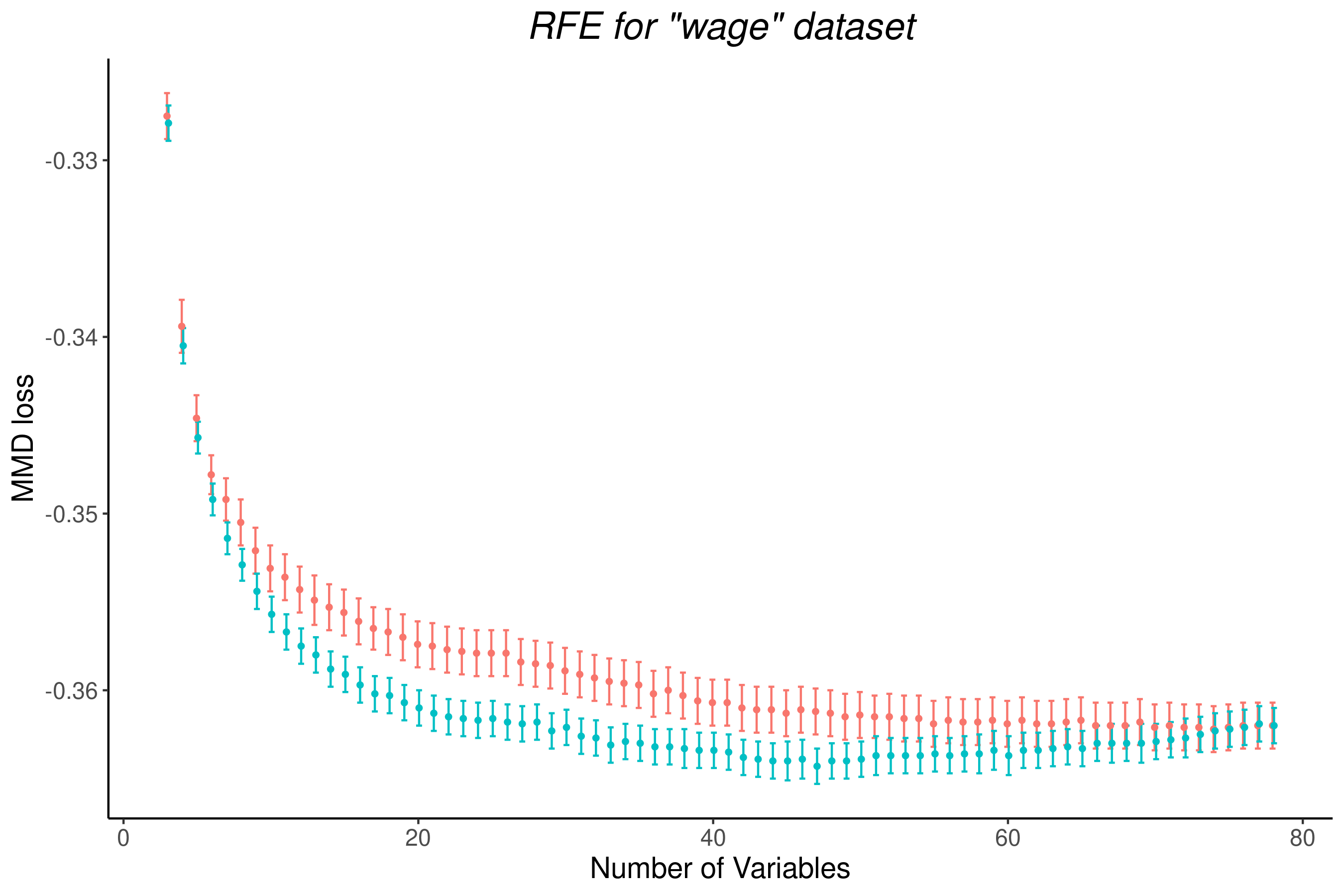}
    \end{center}
    \caption{RFE for `rf1' (left panel) and `wage' (right panel) datasets, using $\VI_n^{(j)}$ (blue) or vimp-drf (red).
    The RFE procedure is repeated $40$ times to compute the standard error of the loss at each step, displayed as error bars.}
    \label{fig:RFE_wage_rf1}
\end{figure*}

In this subsection, we show the efficiency of our importance measure to perform backward variable selection on real data with multivariate outputs. Hence, we use the recursive feature elimination algorithm (RFE), originally introduced by \citet{guyon2002gene} for variable selection with support vector machines, and first adapted to random forests for regression problems by \citet{gregorutti2017correlation}.
The main principle of the RFE is to iteratively remove the less important variable. At each step, the learning algorithm is rerun on the data with the reduced set of inputs, and the importance measure is computed for all variables involved. Then, the less important variable is removed from the data, and the algorithm moves on to the next iteration. In our case, a DRF is fit at each step of the RFE, and we compare our method introduced in Section \ref{sec:vimp}, and vimp-drf. As explained in the introduction, our algorithm focuses on objective (1) of finding a small number of variables with a maximized accuracy. The recursive feature elimination procedure built from $\smash{\mathrm{I}_n^{(j)}}$, removes highly correlated redundant features. Indeed, all variables of a highly correlated and important group have low importance, until only one variable of this group remains in the data over the backward selection, by definition of $\smash{\mathrm{I}_n^{(j)}}$, which estimates the information loss when a given input is removed. To assess the quality of the variable selection, we need an appropriate loss to quantify the accuracy of the empirical conditional measure $\hat{\P}_{\mathbf{Y} \mid \mathbf{X}^{(\mathbf{J})}}$ given by DRF fit with a subset of inputs $\mathbf{J} \subset \{1,\hdots, p\}$, with respect to the theoretical target measure $\P_{\mathbf{Y} \mid \mathbf{X}}$. The MMD provides a natural distance between these two quantities of interest, defined by $\E[\MMD^2(\P_{\mathbf{Y} \mid \mathbf{X}}, \hat{\P}_{\mathbf{Y} \mid \mathbf{X}^{(\mathbf{J})}})]$. We simplify this metric by subtracting $\E[\|\Phi(\P_{\mathbf{Y} \mid \mathbf{X}})\|^2_{\H}]$, since this last term does not rely on the estimated distribution, and we are only interested in relative comparisons of our approach with competitors. Therefore, our evaluation metric writes $\E[\|\mu_{N,n}(\Xbf^{(\mathbf{J})})\|^2_{\H}] - 2 \E[\langle k(\Ybf, \cdot), \mu_{N,n}(\Xbf^{(\mathbf{J})}) \rangle_{\H}]$.
Then our final empirical loss $\smash{\mathcal{L}_n^{(\mathbf{J})}}$ is defined as a Monte-Carlo estimate of the above quantity using an independent sample $(\Xbf'_1, \Ybf'_1), \hdots, (\Xbf'_n, , \Ybf'_n)$, i.e.
\begin{align*}
    \mathcal{L}_n^{(\mathbf{J})} = \frac{1}{n} \sum_{i=1}^n \wbf(\Xbf'_i)^{\top} \Kbf \wbf(\Xbf'_i) - 2\wbf(\Xbf'_i)^{\top} \kbf(\Ybf'_i),
\end{align*}
where the index $\mathbf{J}$ of $\Xbf'^{(\mathbf{J})}_i$ is dropped to lighten notations, and $\kbf(\Ybf'_i)=( k(\Ybf_1,\Ybf'_i) , \ldots, k(\Ybf_{n},\Ybf'_i) )^{\top}$. Notice that this loss can be negative since we subtract a constant positive term, but is obviously still a valid metric for comparisons. Finally, this loss $\smash{\mathcal{L}_n^{(\mathbf{J})}}$ provides a value at each step of the RFE, and to get a unique score, we sum $\smash{\mathcal{L}_n^{(\mathbf{J})}}$ over all steps of the RFE.
Next, we consider a wide variety of real datasets, inspired from \citet{DRF-paper}, which mainly come from \citet{tsoumakas2011mulan}---all dataset details are provided in Appendix A. DRF is fit using $N = 500$ trees, the RFE is stopped when only $3$ variables remain (forests struggle in very small dimensions), the procedure is repeated $40$ times for uncertainties, and each dataset is split in two halves: one to fit the forest and compute the importance measure with out-of-bag predictions, and the other half to estimate the MMD loss.
\begin{table}
\centering
\setlength{\tabcolsep}{3pt}
\begin{tabular}{c c c c c c c}
  \hline
 & $n$ & $p$ & $d$ & $\VI_n^{(j)}$ \tiny{(std)} & vimp-drf \tiny{(std)} \\ 
  \hline
  \textbf{enb} & 768 & 8 & 2 & \textbf{-3.478} \tiny{(0.02)} & -3.105 \tiny{(0.02)} \\ 
  jura & 359 & 15 & 3 & -0.921 \tiny{(0.01)} & -0.906 \tiny{(0.01)} \\ 
  \textbf{wq} & 1060 & 16 & 14 & \textbf{-0.164} \tiny{(0.003)} & -0.155 \tiny{(0.003)} \\ 
  air & 2000 & 15 & 6 & -0.146 \tiny{(0.003)} & -0.143 \tiny{(0.002)} \\ 
  births & 2000 & 24 & 4 & -0.328 \tiny{(0.003)} & -0.329 \tiny{(0.003)} \\ 
  \textbf{rf1} & 2000 & 64 & 8 & \textbf{-7.970} \tiny{(0.04)} & -7.886 \tiny{(0.04)} \\ 
  \textbf{scm20d} & 2000 & 61 & 16 & \textbf{0.289} \tiny{(0.001)} & 0.292 \tiny{(0.001)} \\ 
  \textbf{wage} & 2000 & 78 & 2 & \textbf{-28.032} \tiny{(0.08)} & -27.850 \tiny{(0.1)} \\ 
  \textbf{oes97} & 334 & 263 & 16 & \textbf{1.159} \tiny{(0.004)} & 1.183 \tiny{(0.004)} \\ 
   \hline
\end{tabular}
\caption{Cumulated MMD-loss over RFE steps, using our importance measure $\smash{\VI_n^{(j)}}$ or vimp-drf. Datasets are in bold when the gap between the two losses is higher than the sum of the two standard deviations.}
\label{table:rfe}
\end{table}

Results are displayed in Table \ref{table:rfe} for the considered datasets. In each case, we provide the cumulative MMD-loss, averaged over the $40$ repetitions, along with the standard deviation of this mean value. Table \ref{table:rfe} clearly shows a significant improvement of our importance measure over vimp-drf for most datasets. 
The variable ranking provided by $\smash{\mathrm{I}_n^{(j)}}$ and vimp-drf can be quite different. Considering the dataset `enb' with $8$ inputs for example, the top three variables given by the two methods do not overlap.
Additionally, Figure $\ref{fig:RFE_wage_rf1}$ shows the full path of the MMD-loss at each step of the RFE for several cases. For example, we observe a major improvement in variable selection for `rf1' and `wage' datasets. Figure $1$ in Appendix A displays results for the `jura' dataset, an example where the cumulative loss gap is not really significant, according to Table \ref{table:rfe}. Nevertheless, this figure shows that the loss is flat over all iterations of the RFE, except the last one, which tells us that removing most variables does not hurt DRF estimates, and our method still outperforms vimp-drf at the final step of the RFE.
Overall, our proposed importance measure improves variable selection for all datasets, except ``Births'' and ``Air''. In these two cases, the MMD-loss is constant over most iterations of the RFE, meaning that variables are removed without hurting the quality of the conditional output distribution estimates. This means that few variables are detected as influential by the DRF, making different competitors of equal efficiency to select variables.

\section{CONCLUSION} \label{sec:conclusion}
We have introduced a variable importance algorithm for Distributional Random Forest, which generalizes total Sobol indices using the MMD, to quantify the influence of each input variable on a multivariate output distribution. The method enjoys good theoretical properties with provable consistency, and shows high performance on experiments with both simulated and real data, especially for recursive feature elimination. Besides, the extension of this approach to MMD-based Shapley effects seems an interesting research direction for future work, since Shapley effects are strongly valuable to interpret various learning algorithms, and are so far limited to the detection of inputs impacting output means.

%\subsubsection*{Acknowledgements}

\bibliography{bibfile}

\begin{thebibliography}{}

\bibitem[Athey et~al., 2019]{athey2019generalized}
Athey, S., Tibshirani, J., and Wager, S. (2019).
\newblock {Generalized random forests}.
\newblock {\em The Annals of Statistics}, 47(2):1148--1178.

\bibitem[Auret and Aldrich, 2011]{auret2011empirical}
Auret, L. and Aldrich, C. (2011).
\newblock Empirical comparison of tree ensemble variable importance measures.
\newblock {\em Chemometrics and Intelligent Laboratory Systems}, 105:157--170.

\bibitem[B\'enard et~al., 2022]{SHAFF}
B\'enard, C., Biau, G., Da~Veiga, S., and Scornet, E. (2022).
\newblock Shaff: Fast and consistent shapley effect estimates via random forests.
\newblock In Camps-Valls, G., Ruiz, F. J.~R., and Valera, I., editors, {\em Proceedings of The 25th International Conference on Artificial Intelligence and Statistics}, volume 151 of {\em Proceedings of Machine Learning Research}, pages 5563--5582. PMLR.

\bibitem[Breiman, 2001]{breiman2001random}
Breiman, L. (2001).
\newblock Random forests.
\newblock {\em Machine learning}, 45(1):5--32.

\bibitem[Breiman, 2003]{breiman2003atechnical}
Breiman, L. (2003).
\newblock Setting up, using, and understanding random forests v3.1.
\newblock Technical report, UC Berkeley, Department of Statistics.

\bibitem[Bénard et~al., 2022]{Sobol_MDA}
Bénard, C., Da~Veiga, S., and Scornet, E. (2022).
\newblock {Mean decrease accuracy for random forests: inconsistency, and a practical solution via the Sobol-MDA}.
\newblock {\em Biometrika}, 109(4):881--900.

\bibitem[Candes et~al., 2018]{candes2018panning}
Candes, E., Fan, Y., Janson, L., and Lv, J. (2018).
\newblock Panning for gold:‘model-x’knockoffs for high dimensional controlled variable selection.
\newblock {\em Journal of the Royal Statistical Society Series B: Statistical Methodology}, 80(3):551--577.

\bibitem[{\'C}evid et~al., 2022]{DRF-paper}
{\'C}evid, D., Michel, L., N\"{a}f, J., Meinshausen, N., and B\"{u}hlmann, P. (2022).
\newblock Distributional random forests: Heterogeneity adjustment and multivariate distributional regression.
\newblock {\em Journal of Machine Learning Research}, 23(333):1--79.

\bibitem[Da~Veiga, 2016]{daveiga2016new}
Da~Veiga, S. (2016).
\newblock New perspectives for sensitivity analysis.
\newblock In {\em Proceedings of Mascot-Num 2016 conference}, Toulouse, France.

\bibitem[Da~Veiga, 2021]{daveiga2021kernelbased}
Da~Veiga, S. (2021).
\newblock Kernel-based anova decomposition and shapley effects -- application to global sensitivity analysis.

\bibitem[Genuer et~al., 2010]{Genuer}
Genuer, R., Poggi, J.-M., and Tuleau-Malot, C. (2010).
\newblock Variable selection using random forests.
\newblock {\em Pattern Recognition Letters}, 31(14):2225--2236.

\bibitem[Gregorutti et~al., 2017]{gregorutti2017correlation}
Gregorutti, B., Michel, B., and Saint-Pierre, P. (2017).
\newblock Correlation and variable importance in random forests.
\newblock {\em Statistics and Computing}, 27:659--678.

\bibitem[Gretton et~al., 2007]{gretton2007kernel}
Gretton, A., Borgwardt, K., Rasch, M., Sch{\"o}lkopf, B., and Smola, A.~J. (2007).
\newblock A kernel method for the two-sample-problem.
\newblock In {\em Advances in neural information processing systems}, pages 513--520.

\bibitem[Gretton et~al., 2012]{gretton2012kernel}
Gretton, A., Borgwardt, K.~M., Rasch, M.~J., Sch{\"o}lkopf, B., and Smola, A. (2012).
\newblock A kernel two-sample test.
\newblock {\em The Journal of Machine Learning Research}, 13(1):723--773.

\bibitem[Guyon et~al., 2002]{guyon2002gene}
Guyon, I., Weston, J., Barnhill, S., and Vapnik, V. (2002).
\newblock Gene selection for cancer classification using support vector machines.
\newblock {\em Machine learning}, 46:389--422.

\bibitem[Hooker et~al., 2021]{hooker2019please}
Hooker, G., Mentch, L., and Zhou, S. (2021).
\newblock Unrestricted permutation forces extrapolation: variable importance requires at least one more model, or there is no free variable importance.
\newblock {\em Statistics and Computing}, 31:1--16.

\bibitem[Hsing and Eubank, 2015]{hilbertspacebook}
Hsing, T. and Eubank, R. (2015).
\newblock {\em Theoretical Foundations of Functional Data Analysis, with an Introduction to Linear Operators}.
\newblock Wiley Series in Probability and Statistics. Wiley.

\bibitem[I~Amoukou and Brunel, 2022]{amoukou2022consistent}
I~Amoukou, S. and Brunel, N. (2022).
\newblock Consistent sufficient explanations and minimal local rules for explaining the decision of any classifier or regressor.
\newblock {\em Advances in Neural Information Processing Systems}, 35:8027--8040.

\bibitem[Lei et~al., 2018]{lei2018distribution}
Lei, J., G’Sell, M., Rinaldo, A., Tibshirani, R., and Wasserman, L. (2018).
\newblock Distribution-free predictive inference for regression.
\newblock {\em Journal of the American Statistical Association}, 113(523):1094--1111.

\bibitem[Lin and Jeon, 2006]{lin2006random}
Lin, Y. and Jeon, Y. (2006).
\newblock Random forests and adaptive nearest neighbors.
\newblock {\em Journal of the American Statistical Association}, 101(474):578--590.

\bibitem[Lundberg and Lee, 2017]{lundberg2017unified}
Lundberg, S. and Lee, S.-I. (2017).
\newblock A unified approach to interpreting model predictions.
\newblock In {\em Advances in Neural Information Processing Systems}, volume~30, pages 4765--4774. Curran Associates, Inc.

\bibitem[Mentch and Hooker, 2016]{RFuncertainty}
Mentch, L. and Hooker, G. (2016).
\newblock {Quantifying Uncertainty in Random Forests via Confidence Intervals and Hypothesis Tests}.
\newblock {\em Journal of Machine Learning Research}, 17(1):841--881.

\bibitem[Michel and {\'C}evid, 2021]{drf-package}
Michel, L. and {\'C}evid, D. (2021).
\newblock {\em drf: {D}istributional {R}andom {F}orests}.
\newblock \textsf{R} package version 1.1.0.

\bibitem[Näf et~al., 2023]{näf2023confidence}
Näf, J., Emmenegger, C., Bühlmann, P., and Meinshausen, N. (2023).
\newblock Confidence and uncertainty assessment for distributional random forests.
\newblock {\em Preprint arXiv:2302.05761}.

\bibitem[Owen, 2014]{owen2014sobol}
Owen, A. (2014).
\newblock Sobol'indices and {S}hapley value.
\newblock {\em SIAM/ASA Journal on Uncertainty Quantification}, 2:245--251.

\bibitem[Rasmussen and Williams, 2006]{rasmussen2006}
Rasmussen, C.~E. and Williams, C. K.~I. (2006).
\newblock {\em Gaussian processes for machine learning.}
\newblock Adaptive computation and machine learning. MIT Press.

\bibitem[Scornet, 2016]{scornet2016random}
Scornet, E. (2016).
\newblock Random forests and kernel methods.
\newblock {\em IEEE Transactions on Information Theory}, 62(3):1485--1500.

\bibitem[Scornet, 2023]{scornet2023trees}
Scornet, E. (2023).
\newblock Trees, forests, and impurity-based variable importance in regression.
\newblock {\em Annales de l'Institut Henri Poincare (B) Probabilites et statistiques}, 59:21--52.

\bibitem[Scornet et~al., 2015]{scornet2015consistency}
Scornet, E., Biau, G., and Vert, J.-P. (2015).
\newblock Consistency of random forests.
\newblock {\em The Annals of Statistics}, 43:1716--1741.

\bibitem[Segal and Xiao, 2011]{segal2011multivariate}
Segal, M. and Xiao, Y. (2011).
\newblock {M}ultivariate random forests.
\newblock {\em Wiley Interdisciplinary Reviews: Data Mining and Knowledge Discovery}, 1(1):80--87.

\bibitem[Sikdar et~al., 2023]{MRFVI}
Sikdar, S., Hooker, G., Kadiyali, V., et~al. (2023).
\newblock Variable importance measures for variable selection and statistical inference in multivariate random forests.
\newblock PREPRINT (Version 1) available at Research Square.

\bibitem[Simon-Gabriel et~al., 2020]{simon2020metrizing}
Simon-Gabriel, C.-J., Barp, A., and Mackey, L. (2020).
\newblock Metrizing weak convergence with maximum mean discrepancies.
\newblock {\em Preprint arXiv:2006.09268}.

\bibitem[Sobol, 1993]{Sobol_indices}
Sobol, I.~M. (1993).
\newblock Sensitivity estimates for nonlinear mathematical models.
\newblock {\em Mathematical Modelling and Computational Experiments}, 1:407--414.

\bibitem[Sriperumbudur, 2016]{optimalestimationofprobabilitymeasures}
Sriperumbudur, B. (2016).
\newblock {On the optimal estimation of probability measures in weak and strong topologies}.
\newblock {\em Bernoulli}, 22(3):1839--1893.

\bibitem[Strobl et~al., 2008]{strobl2008conditional}
Strobl, C., Boulesteix, A.-L., Kneib, T., Augustin, T., and Zeileis, A. (2008).
\newblock Conditional variable importance for random forests.
\newblock {\em BMC Bioinformatics}, 9:307.

\bibitem[Strobl et~al., 2007]{strobl2007bias}
Strobl, C., Boulesteix, A.-L., Zeileis, A., and Hothorn, T. (2007).
\newblock Bias in random forest variable importance measures: illustrations, sources and a solution.
\newblock {\em BMC Bioinformatics}, 8:25.

\bibitem[Tsoumakas et~al., 2011]{tsoumakas2011mulan}
Tsoumakas, G., Spyromitros-Xioufis, E., Vilcek, J., and Vlahavas, I. (2011).
\newblock Mulan: A java library for multi-label learning.
\newblock {\em The Journal of Machine Learning Research}, 12:2411--2414.

\bibitem[Wager and Athey, 2017]{wager2017estimation}
Wager, S. and Athey, S. (2017).
\newblock Estimation and inference of heterogeneous treatment effects using random forests.
\newblock {\em Preprint arXiv:1510.04342}.

\bibitem[Wager and Athey, 2018]{wager2018estimation}
Wager, S. and Athey, S. (2018).
\newblock Estimation and inference of heterogeneous treatment effects using random forests.
\newblock {\em Journal of the American Statistical Association}, 113(523):1228--1242.

\bibitem[Williamson et~al., 2022]{Williamson}
Williamson, B.~D., Gilbert, P.~B., Simon, N.~R., and Carone, M. (2022).
\newblock A general framework for inference on algorithm-agnostic variable importance.
\newblock {\em Journal of the American Statistical Association}, 0(0):1--14.

\end{thebibliography}

%%%%%%%%%%%%%%%%%%%%%%%%%%%%%%%%%%%%%%%%%%%%%%%%%%%%%%%%%%%%
\section*{Checklist}

\begin{enumerate}

 \item For all models and algorithms presented, check if you include:
 \begin{enumerate}
   \item A clear description of the mathematical setting, assumptions, algorithm, and/or model. [Yes]
   \item An analysis of the properties and complexity (time, space, sample size) of any algorithm. [Yes]
   \item (Optional) Anonymized source code, with specification of all dependencies, including external libraries. [Yes]
 \end{enumerate}

 \item For any theoretical claim, check if you include:
 \begin{enumerate}
   \item Statements of the full set of assumptions of all theoretical results. [Yes]
   \item Complete proofs of all theoretical results. [Yes]
   \item Clear explanations of any assumptions. [Yes]     
 \end{enumerate}

 \item For all figures and tables that present empirical results, check if you include:
 \begin{enumerate}
   \item The code, data, and instructions needed to reproduce the main experimental results (either in the supplemental material or as a URL). [Yes]
   \item All the training details (e.g., data splits, hyperparameters, how they were chosen). [Yes]
         \item A clear definition of the specific measure or statistics and error bars (e.g., with respect to the random seed after running experiments multiple times). [Yes]
         \item A description of the computing infrastructure used. (e.g., type of GPUs, internal cluster, or cloud provider). [Yes]
 \end{enumerate}

 \item If you are using existing assets (e.g., code, data, models) or curating/releasing new assets, check if you include:
 \begin{enumerate}
   \item Citations of the creator If your work uses existing assets. [Yes]
   \item The license information of the assets, if applicable. [Not Applicable]
   \item New assets either in the supplemental material or as a URL, if applicable. [Not Applicable]
   \item Information about consent from data providers/curators. [Not Applicable]
   \item Discussion of sensible content if applicable, e.g., personally identifiable information or offensive content. [Not Applicable]
 \end{enumerate}

 \item If you used crowdsourcing or conducted research with human subjects, check if you include:
 \begin{enumerate}
   \item The full text of instructions given to participants and screenshots. [Not Applicable]
   \item Descriptions of potential participant risks, with links to Institutional Review Board (IRB) approvals if applicable. [Not Applicable]
   \item The estimated hourly wage paid to participants and the total amount spent on participant compensation. [Not Applicable]
 \end{enumerate}

 \end{enumerate}

\appendix
\onecolumn

\vspace*{1mm}

\begin{center}
\textbf{\Large Supplementary Material for ``MMD-based Variable Importance for Distributional Random Forest''}
\end{center}

\vspace*{1cm}

% \aistatstitle{Supplementary Material for ``MMD-based Variable Importance for Distributional Random Forest''}

\FloatBarrier

Appendix A provides additional experiments and details for Section $4$ of the main article. Then, Appendix B states the proofs of all theorems and propositions.

\section{Experiments} \label{appendix:xp}

The first subsection is dedicated to additional details for experiments with simulated data, while the second subsection focuses on RFE for real datasets. In particular, we provide figures of the RFE path for the nine datasets, along with data sources. 

\subsection{Simulated data}

For simulated data, all experiment settings are given in the main article. For the univariate case, we provide an extended version of Table \ref{table:xp_univariate} with standard deviations of the mean importance values in small font. Notice that these uncertainties are small, and are thus omitted in the paper.

\begin{table}
\setlength{\tabcolsep}{1pt}
\centering
\begin{tabular}{c c}
  \hline
   $X^{(j)}$ & $\VI_n^{(j)}$ \\ 
  \hline
  \blueft{$X^{(1)}$} & 0.181 \tiny{0.009} \\ 
  \blueft{$X^{(4)}$} & 0.073 \tiny{0.007} \\ 
  \blueft{$X^{(5)}$} & 0.073 \tiny{0.008} \\ 
  \blueft{$X^{(2)}$} & 0.072 \tiny{0.004} \\ 
  \blueft{$X^{(3)}$} & 0.065 \tiny{0.004} \\ 
  $X^{(10)}$ & 0.010 \tiny{0.001} \\ 
  $X^{(7)}$ & 0.005 \tiny{0.0004} \\ 
  $X^{(6)}$ & 0.005 \tiny{0.0005} \\ 
  $X^{(9)}$ & 0.005 \tiny{0.0003} \\ 
  $X^{(8)}$ & 0.005 \tiny{0.0002} \\ 
   \hline
\end{tabular} \hspace*{2mm}
\begin{tabular}{c c}
  \hline
   $X^{(j)}$ & v-drf \\ 
  \hline
  \blueft{$X^{(1)}$} & 0.649 \tiny{0.01} \\ 
  $X^{(10)}$ & 0.096 \tiny{0.01} \\ 
  \blueft{$X^{(2)}$} & 0.062 \tiny{0.004} \\ 
  \blueft{$X^{(5)}$} & 0.059 \tiny{0.007} \\ 
  \blueft{$X^{(4)}$} & 0.056 \tiny{0.005} \\ 
  \blueft{$X^{(3)}$} & 0.050 \tiny{0.003} \\ 
  $X^{(6)}$ & 0.007 \tiny{0.0007} \\ 
  $X^{(9)}$ & 0.007 \tiny{0.0005} \\ 
  $X^{(7)}$ & 0.007 \tiny{0.0007} \\ 
  $X^{(8)}$ & 0.007 \tiny{0.0005} \\ 
   \hline
\end{tabular}  \hspace*{2mm}
\begin{tabular}{c c}
  \hline
   $X^{(j)}$ & MDA \\ 
  \hline
  \blueft{$X^{(1)}$} & 6.47 \tiny{0.6} \\ 
  $X^{(10)}$ & 2.56 \tiny{0.5} \\ 
  \blueft{$X^{(2)}$} & 1.33 \tiny{0.4} \\ 
  $X^{(6)}$ & 0.25 \tiny{0.1} \\ 
  $X^{(8)}$ & 0.24 \tiny{0.2} \\ 
  $X^{(7)}$ & 0.18 \tiny{0.2} \\ 
  $X^{(9)}$ & 0.18 \tiny{0.1} \\ 
  \blueft{$X^{(3)}$} & -0.52 \tiny{0.2} \\ 
  \blueft{$X^{(5)}$} & -0.62 \tiny{0.1} \\ 
  \blueft{$X^{(4)}$} & -0.90 \tiny{0.1} \\ 
   \hline
\end{tabular}  \hspace*{2mm}
\begin{tabular}{c c}
  \hline
   $X^{(j)}$ & S-MDA \\ 
  \hline
  \blueft{$X^{(1)}$} & 0.014 \tiny{0.003} \\ 
  \blueft{$X^{(2)}$} & 0.012 \tiny{0.003} \\ 
  $X^{(8)}$ & -0.001 \tiny{0.001} \\ 
  $X^{(7)}$ & -0.003 \tiny{0.001} \\ 
  $X^{(9)}$ & -0.003 \tiny{0.0006} \\ 
  $X^{(6)}$ & -0.003 \tiny{0.0007} \\ 
  \blueft{$X^{(5)}$} & -0.003 \tiny{0.002} \\ 
  \blueft{$X^{(3)}$} & -0.004 \tiny{0.001} \\ 
  \blueft{$X^{(4)}$} & -0.004 \tiny{0.002} \\ 
  $X^{(10)}$ & -0.006 \tiny{0.001} \\ 
   \hline
\end{tabular}
\caption{Variable importance for the univariate output experiment for $\smash{\VI_n^{(j)}}$, vimp-drf (v-drf), MDA, and Sobol-MDA (S-MDA), with standard deviations of mean importance values in small font.}
\label{table_app:xp_univariate}
\end{table}

\subsection{RFE for real datasets}

We first recall Table $5$ of the main article in Table \ref{table_app:rfe} below. The choice of these datasets was driven by the original DRF paper \citep{DRF-paper}, originally inspired by \citet{tsoumakas2011mulan}, which provides data with multivariate outputs. The data was found in the following repository : \url{https://github.com/tsoumakas/mulan}, accessed in August 2023.
For the details and sources of ``Births'', ``Air quality'', and ``wage'' datasets, see Appendix C in \citet{DRF-paper}. Notice that our ``Births'' data corresponds to ``birth2'' in  \citet{DRF-paper}.
When the sample size exceeds $n = 2000$, a random subsampling of $2000$ observations is done at each repetition of the RFE, to keep a reasonable computational cost.
The RFE experiments were conducted on a cluster with the following characteristics: 27 computational nodes Ice Lake with 48 cores (CPU: Intel Xeon Gold 6342, 2x24 cores, 2.80 Ghz. Memory : 263 GB (5.4 GB/cores)).

Then, Figure \ref{fig:RFE_enb_jura} provides the RFE paths for ``enb'' and ``jura'' datasets, Figure \ref{fig:RFE_wq_scm20d} for ``wq'' and ``scm20d'' datasets, Figure \ref{fig:RFE_births_oes97} for ``Births'' and ``oes97'' datasets, and Figure \ref{fig:RFE_air} for ``Air quality'' dataset. Overall, our proposed importance measure improves variable selection for all datasets, except ``Births'' and ``Air''. In these two cases, the MMD-loss is constant over most iterations of the RFE, meaning that variables are removed without hurting the quality of the conditional output distribution estimates. This means that few variables are detected as influential by the DRF, making different competitors of equal efficiency to select variables.

\begin{table}
\centering
\setlength{\tabcolsep}{2pt}
\begin{tabular}{c c c c c c c c}
  \hline
 & $n$ & $p$ & $d$ & $\VI_n^{(j)}$ & std & vimp-drf & std \\ 
  \hline
  \textbf{enb} & 768 & 8 & 2 & \textbf{-3.478} & 0.020 & -3.105 & 0.017 \\ 
  jura & 359 & 15 & 3 & -0.921 & 0.014 & -0.906 & 0.014 \\ 
  \textbf{wq} & 1060 & 16 & 14 & \textbf{-0.164} & 0.003 & -0.155 & 0.003 \\ 
  air & 2000 & 15 & 6 & -0.146 & 0.003 & -0.143 & 0.002 \\ 
  births & 2000 & 24 & 4 & -0.328 & 0.003 & -0.329 & 0.003 \\ 
  \textbf{rf1} & 2000 & 64 & 8 & \textbf{-7.970} & 0.037 & -7.886 & 0.035 \\ 
  \textbf{scm20d} & 2000 & 61 & 16 & \textbf{0.289} & 0.001 & 0.292 & 0.001 \\ 
  \textbf{wage} & 2000 & 78 & 2 & \textbf{-28.032} & 0.077 & -27.850 & 0.098 \\ 
  \textbf{oes97} & 334 & 263 & 16 & \textbf{1.159} & 0.004 & 1.183 & 0.004 \\ 
   \hline
\end{tabular}
\caption{Cumulated MMD-loss over RFE steps, using our importance measure $\smash{\VI_n^{(j)}}$ or vimp-drf. Datasets are in bold when the gap between the two losses is higher than the sum of the two standard deviations.}
\label{table_app:rfe}
\end{table}

\begin{figure}
    \begin{center}
        \includegraphics[width=0.5 \textwidth]{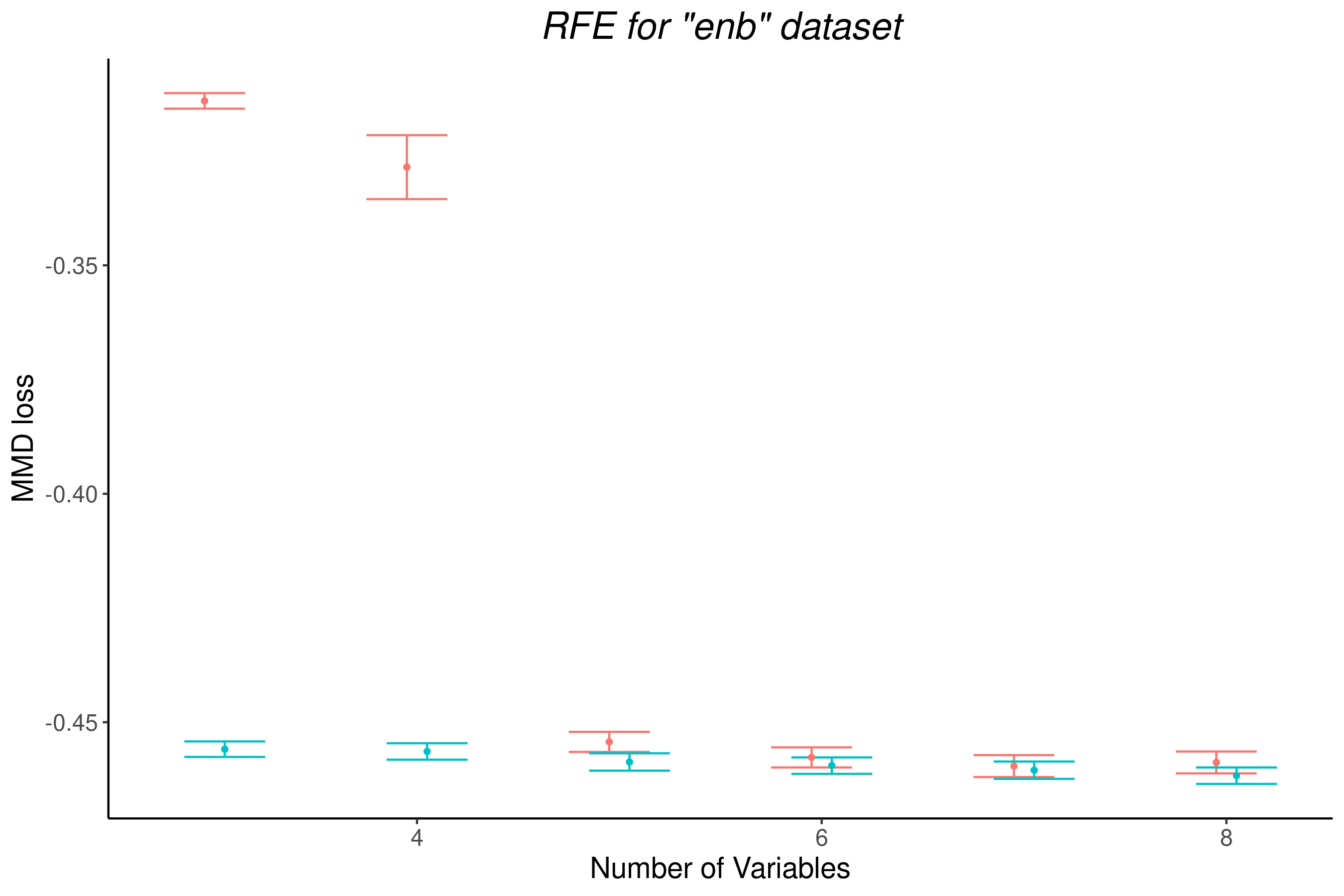} \hspace*{-3mm}
        \includegraphics[width=0.5 \textwidth]{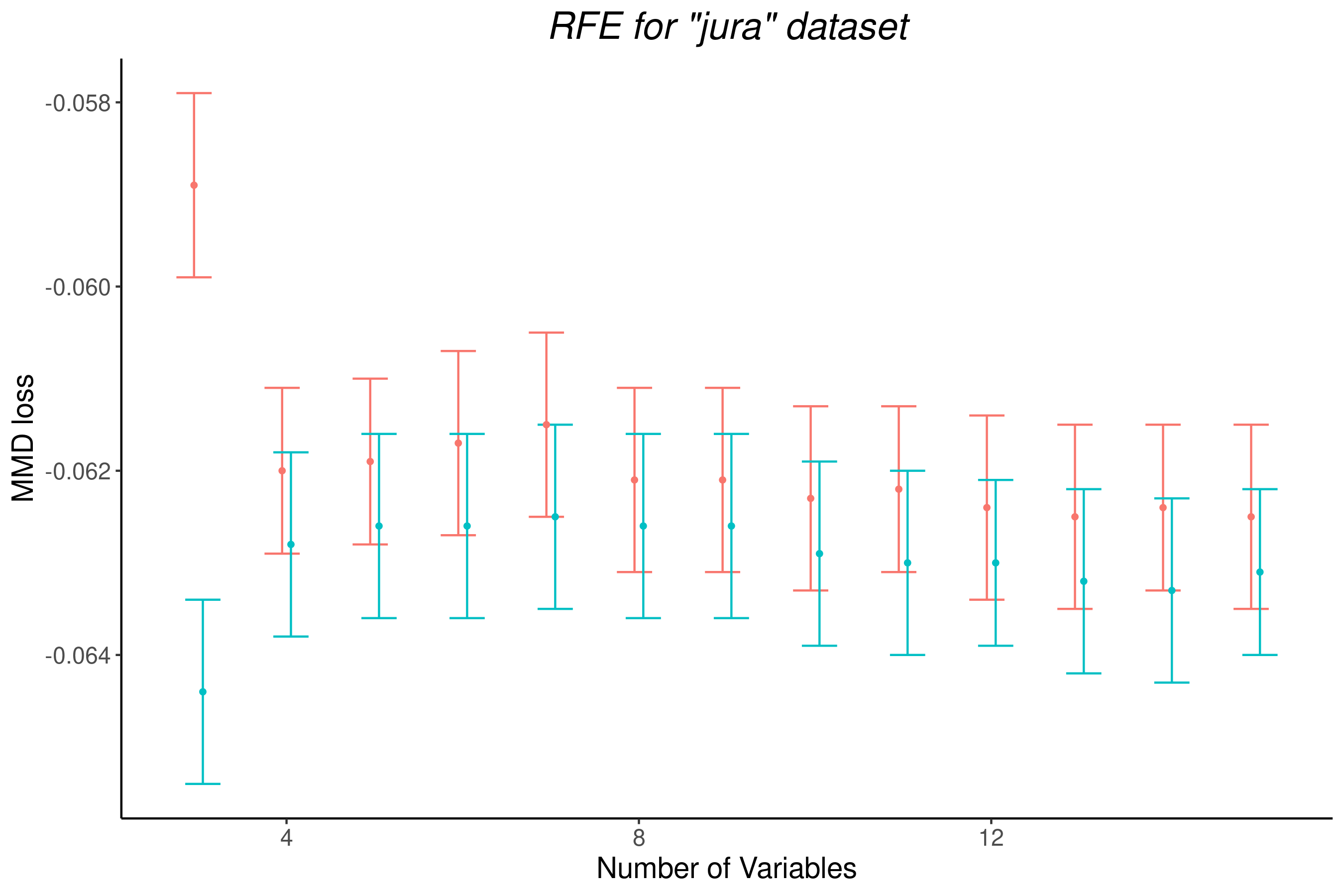}
    \end{center}
    \caption{RFE for `enb' (left panel) and `jura' (right panel) datasets, using our importance measure $\VI_n^{(j)}$ (blue) or vimp-drf (red).}
    \label{fig:RFE_enb_jura}
\end{figure}

\begin{figure}
    \begin{center}
        \includegraphics[width=0.5 \textwidth]{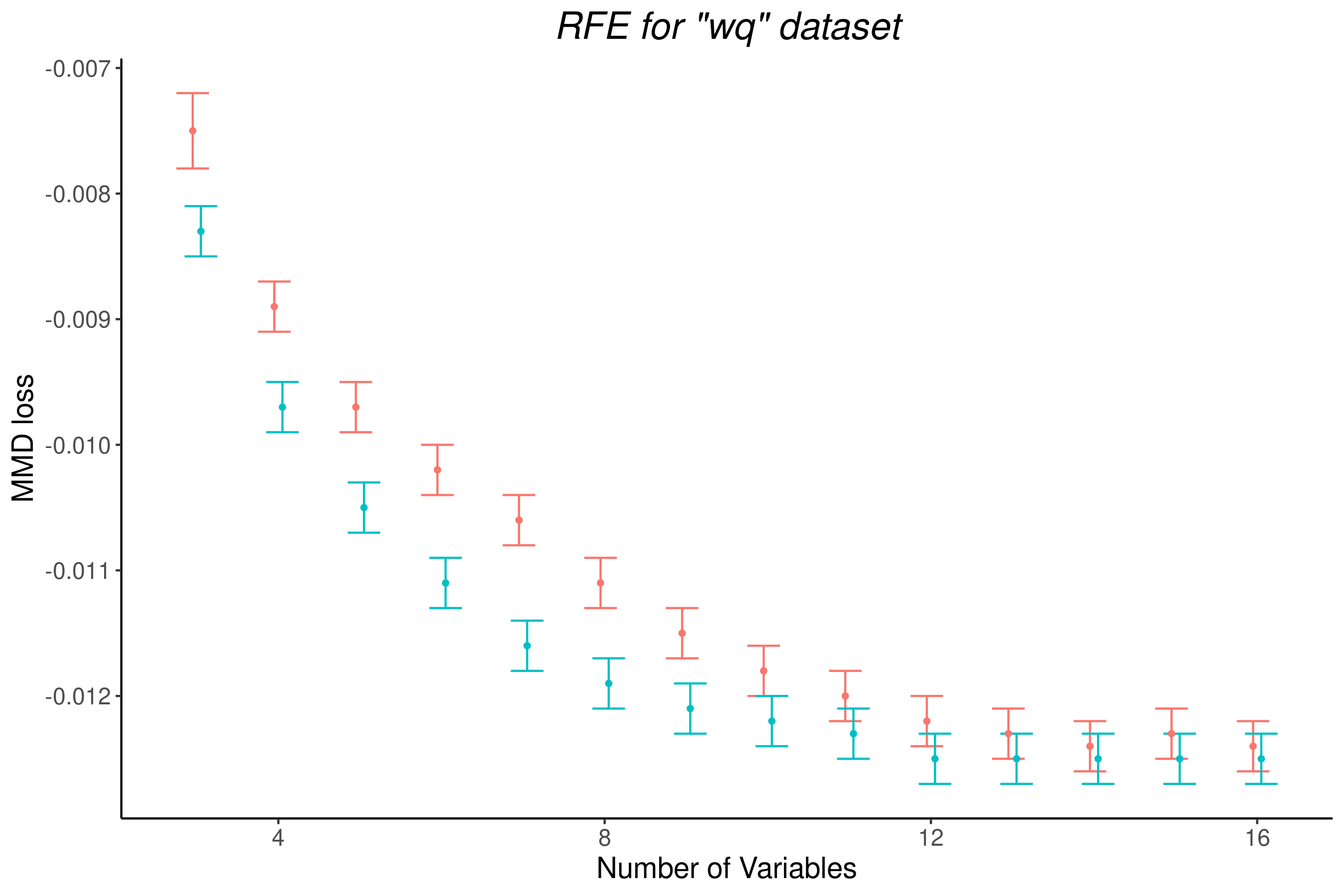} \hspace*{-3mm}
        \includegraphics[width=0.5 \textwidth]{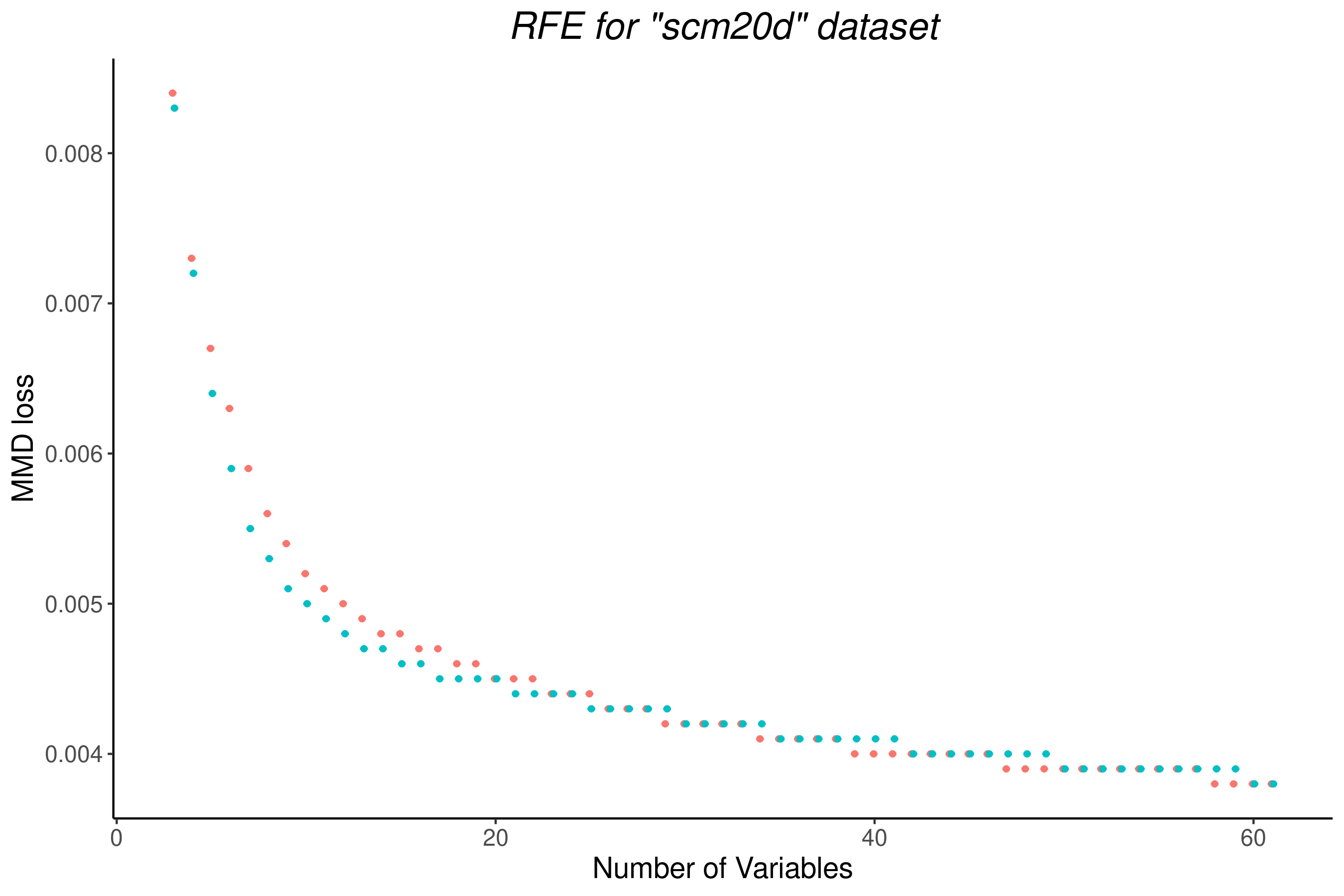}
    \end{center}
    \caption{RFE for `wq' (left panel) and `scm20d' (right panel) datasets, using our importance measure $\VI_n^{(j)}$ (blue) or vimp-drf (red).}
    \label{fig:RFE_wq_scm20d}
\end{figure}

\begin{figure}
    \begin{center}
        \includegraphics[width=0.5 \textwidth]{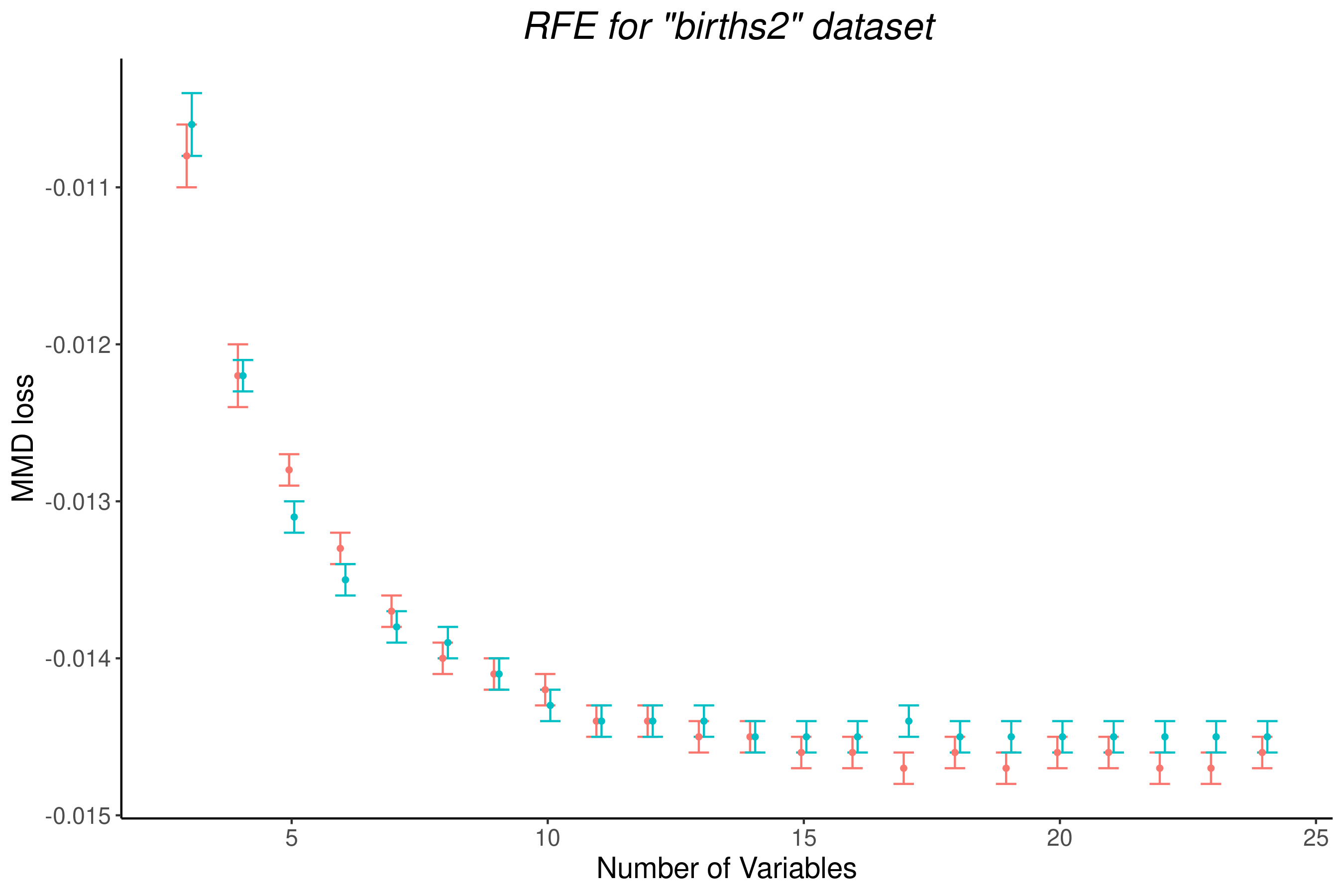} \hspace*{-3mm}
        \includegraphics[width=0.5 \textwidth]{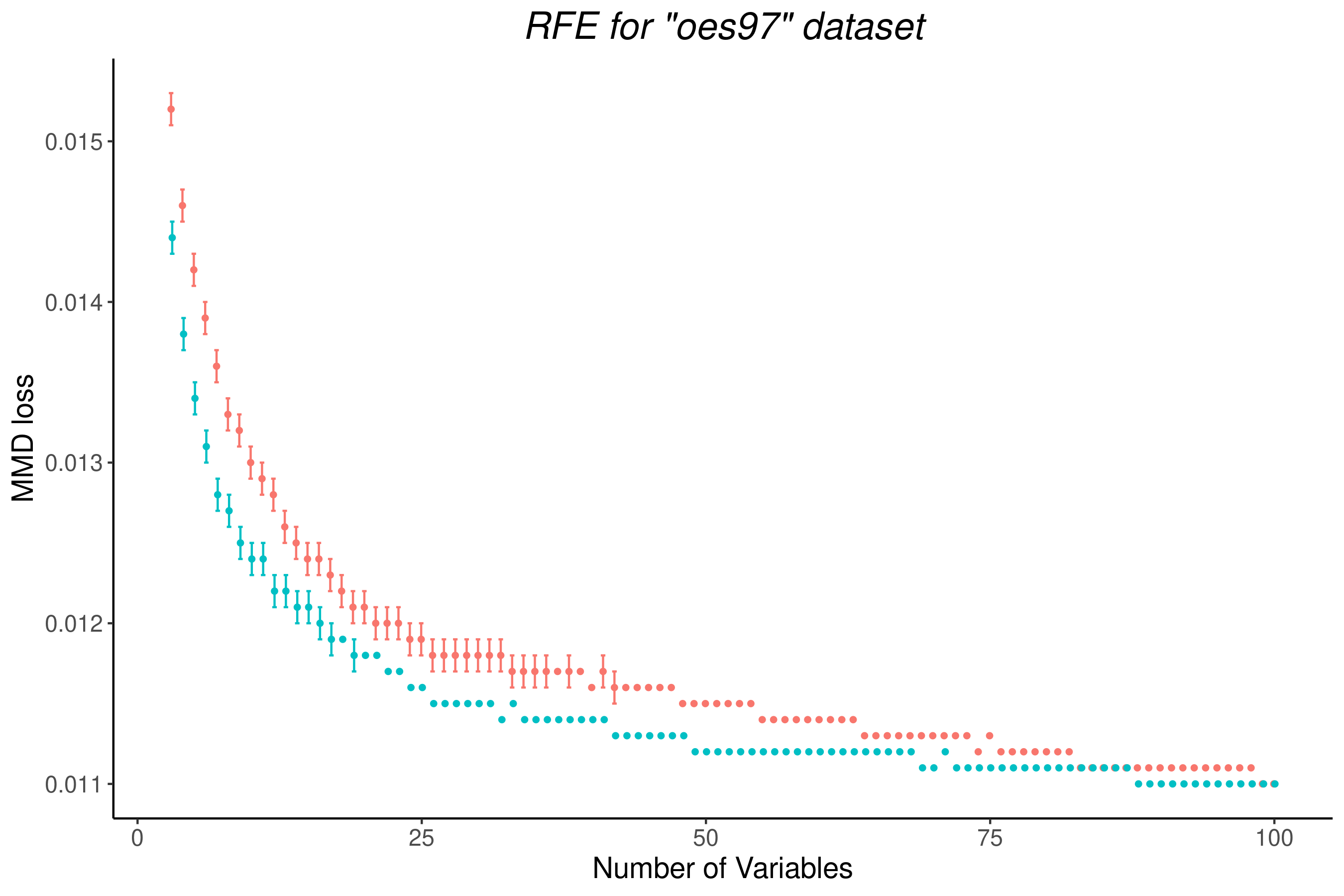}
    \end{center}
    \caption{RFE for `Births' (left panel) and `oes97' (right panel) datasets, using our importance measure $\VI_n^{(j)}$ (blue) or vimp-drf (red).}
    \label{fig:RFE_births_oes97}
\end{figure}

\begin{figure}
    \begin{center}
        \includegraphics[width=0.5 \textwidth]{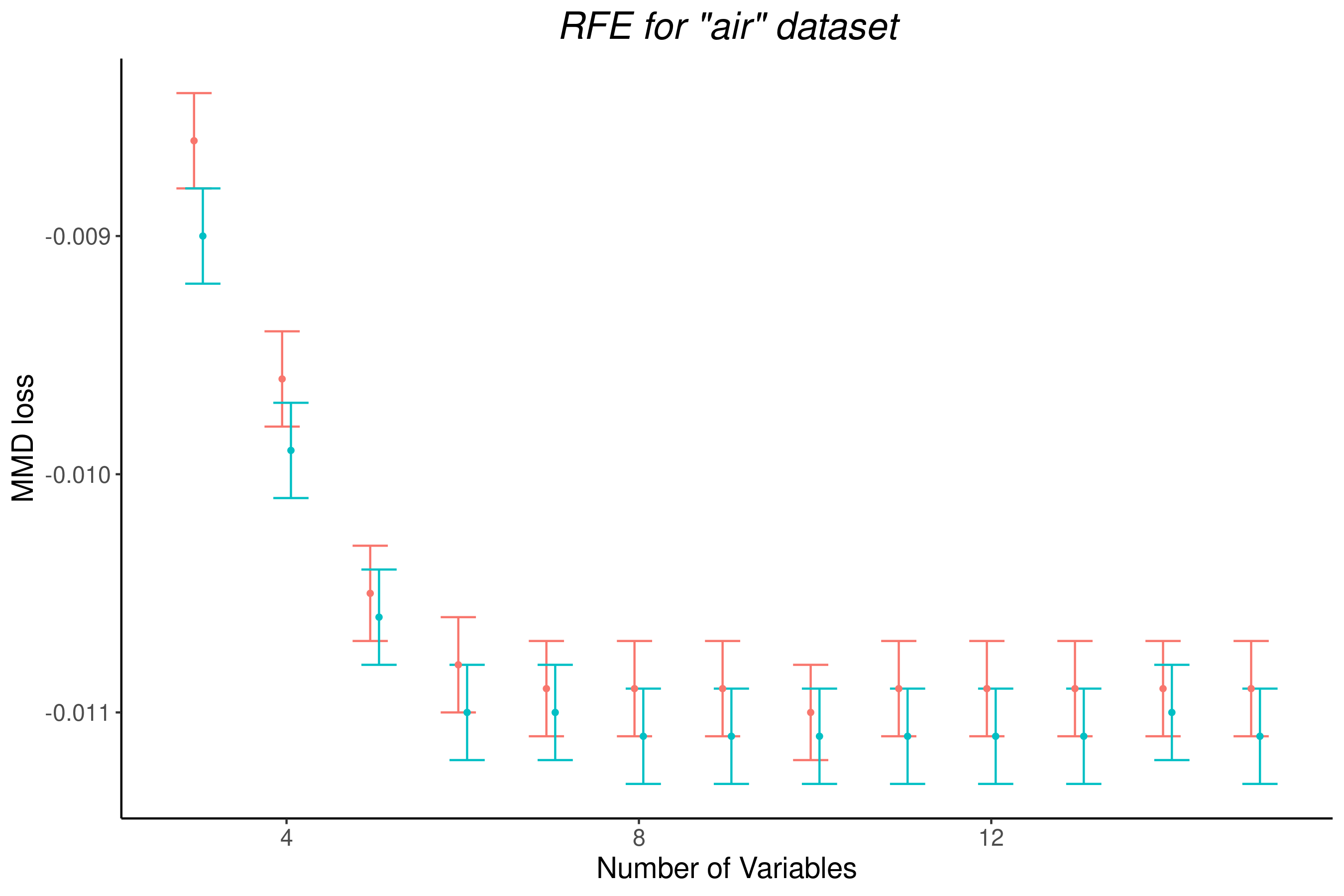} \hspace*{-3mm}
    \end{center}
    \caption{RFE for `Air' dataset, using our importance measure $\VI_n^{(j)}$ (blue) or vimp-drf (red).}
    \label{fig:RFE_air}
\end{figure}

\section{Proofs}\label{appendix: proofs}

We first recall the main equations of the article.
\begin{align}\label{eq_app:DRFvarimportance}
    \VI^{(j)} \defeq \frac{\E[ \VH[\mu(\Xbf) \mid \Xbf^{(-j)}]]}{\VH[\mu(\Xbf)]},
\end{align}
\begin{align}\label{eq_app:Injdef}
    \VI_n^{(j)} = \frac{ \sum_{i=1}^n \norm{\mu_{N,n}(\Xbf'_i) - \mu_{N,n}(\Xbf_i'^{(-j)}) }_{\H}^2 }{\sum_{i=1}^n \norm{ \mu_{N,n}(\Xbf'_i) - \overline{\mu_{N,n}} }_{\H}^2} - \VI_n^{(0)},
\end{align}
\begin{equation}\label{eq_app:finalestimator}
    \hmun = \binom{n}{s_n}^{-1} \hspace*{-5mm} \sum_{i_1 < \cdots < i_{s_n}} \hspace*{-4mm} \E\left[T_n(\xbf; \varepsilon, \{\mathbf{Z}_{i_1}, \ldots, \mathbf{Z}_{i_{s_n}}\}) \mid \Zcal_n \right].
\end{equation}

We put the following assumptions on the forest construction.
\begin{enumerate}[label=(\textbf{F\arabic*})]
    \item\label{forestass1} (\textit{Honesty}) The data used for constructing each tree is split into two halves; the first is used for determining the splits and the second for populating the leaves and thus for estimating the response. The covariates in the second sample may be used for the splits, to enforce the subsequent assumptions, but not the response.
    \item\label{forestass2} (\textit{Random-split}) At every split point and for all feature dimensions $j=1,\ldots,p$, the probability that the split occurs along the feature $X_j$ is bounded from below by $\pi/p$ for some $\pi > 0$.
    \item\label{forestass3} (\textit{Symmetry}) The (randomized) output of a tree does not depend on the ordering of the training samples.
    \item\label{forestass4} (\textit{$\alpha$-regularity}) 
   After splitting a parent node, each child node contains at least a fraction $\alpha \leq 0.2$ of the parent's training samples. Moreover, the trees are grown until every leaf contains between $\kappa$ and $2\kappa - 1$ many observations for some fixed tuning parameter $\kappa \in \N$. 
    \item\label{forestass5}(\textit{Data sampling}) 
    To grow a tree, a subsample of size $s_n$ out of the $n$ training data points is sampled. We consider $s_n=n^{\beta}$ with $0 < \beta < 1$.
\end{enumerate}

We also recall the assumptions of the article.
\begin{enumerate}[label=(\textbf{K\arabic*})]
    \item\label{kernelass1_app} The kernel $k$ is bounded, and the function $(\xbf,\ybf) \mapsto k(\xbf,\ybf)$ is (jointly) continuous.
    \item\label{kernelass3_app} The kernel $k$ is characteristic.
\end{enumerate}

\begin{enumerate}[label=(\textbf{D\arabic*})]
    \item\label{dataass1_app} The observations $\mathbf{X}_1,\ldots, \mathbf{X}_n$ are independent and identically distributed on $[0,1]^p$, with a density bounded from below and above by strictly positive constants.
    \item\label{dataass2_app} The mapping $\xbf \mapsto \mu(\xbf)=\E[ k(\Ybf,\cdot)  \mid \Xbf \myeq \xbf] \in \H $ is Lipschitz.
    \item\label{dataass3_app} The mapping $\xbf^{(-j)} \mapsto \E[ k(\Ybf,\cdot)  \mid \Xbf^{(-j)} \myeq \xbf^{(-j)}] \in \H $ is Lipschitz.
\end{enumerate}

For a sequence of random variables $X_n\colon \Omega \to \R$ and $a_n \in (0,+\infty)$, $n \in \N$, we write $X_n=\O_p(a_n)$ if 
\[
\lim_{M \to \infty} \sup_{n} \P(a_n^{-1} |X_n| > M) = 0.
\]

\VIMMD*

\begin{proof}
    By definition of our importance measure given in Equation \eqref{eq_app:DRFvarimportance}, we have 
    \begin{align*}
        \VI^{(j)} \defeq \frac{\E[ \VH[\mu(\Xbf) \mid \Xbf^{(-j)}]]}{\VH[\mu(\Xbf)]}.
    \end{align*}
    We first consider the numerator, and write
    \begin{align*}
        \E[\VH[\mu(\Xbf) \mid \Xbf^{(-j)}]] &= \E[\|\mu(\Xbf) - \E[\mu(\Xbf) \mid \Xbf^{(-j)}]\|^2_{\H}] 
        = \E[\|\E[k(\Ybf,\cdot) \mid \Xbf] - \E[k(\Ybf,\cdot) \mid \Xbf^{(-j)}]\|^2_{\H}] \\ 
        &= \E[\|\Phi(\PYgX) - \Phi(\PYgXmj)\|^2_{\H}] = \E[\MMD^2(\PYgX, \PYgXmj)].        
    \end{align*} 
    For the denominator, we have
    \begin{align*}
        \VH[\mu(\Xbf)] &= \E[\|\mu(\Xbf) - \E[\mu(\Xbf)]\|^2_{\H}] = \E[\|\mu(\Xbf)\|^2_{\H}] - \|\E[\mu(\Xbf)]\|^2_{\H}
        \\ &= \E[\|\E[k(\Ybf,\cdot) \mid \Xbf]\|^2_{\H}] - \|\E[k(\Ybf,\cdot)]\|^2_{\H}.        
    \end{align*}
    Using the RKHS properties of $\H$, and with $\Ybf'$ an independent copy of $\Ybf$, we can write
    \begin{align*}
        \|\E[k(\Ybf,\cdot)]\|^2_{\H} = \langle \E[k(\Ybf,\cdot)], \E[k(\Ybf',\cdot)] \rangle_{\H} = \E[\langle k(\Ybf,\cdot), k(\Ybf',\cdot) \rangle_{\H}] = \E[k(\Ybf,\Ybf')].
    \end{align*}
   For the other term, we introduce $\tilde{\Ybf}$ distributed as $\Ybf$, and independent and identically distributed as $\Ybf$ conditional on $\Xbf$, then
    \begin{align*}
        \|\E[k(\Ybf,\cdot) \mid \Xbf]\|^2_{\H} = \langle \E[k(\Ybf,\cdot)\mid \Xbf], \E[k(\tilde{\Ybf},\cdot) \mid \Xbf] \rangle_{\H} = \E[\langle k(\Ybf,\cdot), k(\tilde{\Ybf},\cdot) \rangle_{\H} \mid \Xbf] = \E[k(\Ybf, \tilde{\Ybf}) \mid \Xbf].
    \end{align*}
    Overall, $\VH[\mu(\Xbf)] =  \E[k(\Ybf, \tilde{\Ybf})] - \E[k(\Ybf,\Ybf')]$. Next, by definition of the MMD distance, we have
    \begin{align*}
        \E[\MMD^2(\PY, \PYgX)] = \E[k(\Ybf,\Ybf')] + \E[k(\Ybf,\tilde{\Ybf})] - 2 \E[k(\Ybf',\tilde{\Ybf})],
    \end{align*}
    and since $\Ybf'$ and $\Xbf$ are independent, $\E[k(\Ybf',\tilde{\Ybf})] = \E[k(\Ybf',\Ybf)]$, and we get 
     \begin{align*}
        \E[\MMD^2(\PY, \PYgX)] = \E[k(\Ybf, \tilde{\Ybf})] - \E[k(\Ybf,\Ybf')] = \VH[\mu(\Xbf)].
    \end{align*}
    Finally, we obtain  
    \begin{align*}
        \VI^{(j)} &= \frac{\E[\MMD^2(\PYgX, \PYgXmj)]}{\E[\MMD^2(\PY, \PYgX)]}.
    \end{align*}
    Similarly, it holds that 
    \begin{align*}
        \E[\MMD^2(\PYgX, \PYgXmj)] = \E[k(\Ybf, \tilde{\Ybf})] - \E[k(\Ybf,\tilde{\Ybf}^{(-j)})],
    \end{align*}
    where $\tilde{\Ybf}^{(-j)}$ is distributed as $\Ybf$, and independent and identically distributed as $\Ybf$ conditional on $\Xbf^{(-j)}$. Consequently, we get
    \begin{align*}
        \E[\MMD^2(\PY, \PYgX)] - \E[\MMD^2(\PY, \PYgXmj)] = \E[\MMD^2(\PYgX, \PYgXmj)],
    \end{align*}
    and we also obtain
    \begin{align*}
        \VI^{(j)} &= 1 - \frac{\E[\MMD^2(\PY, \PYgXmj)]}{\E[\MMD^2(\PY, \PYgX)]}.
    \end{align*}
\end{proof}

\VIpositive*

\begin{proof}
    The condition that $\PYgx \neq \PYgxmj$ over a set of non-null probability, combined with Assumption \ref{kernelass3_app}, implies that $\norm{\mu(\xbf)-\mu(\xbf^{(-j)})}_{\H} > 0 $, for $\xbf$ in a set with nonzero probability. This in turn implies $\E[\| \mu(\Xbf) - \mu(\Xbf^{(-j)}) \|_{\H}^2] > 0$, and thus $\VI^{(j)} > 0$.
    Additionally, $\VI^{(j)} \leq 1$ is a direct consequence of Proposition \ref{prop:VI_MMD}, since MMD takes non-negative values.
\end{proof}

\meanconsistency*

\begin{proof}
    \citet[Theorem 1]{DRF-paper} implies that
    \begin{equation}\label{eq: rate} 
\norm{\hmun- \mu(\xbf)}_\H \stackrel{p}{\to} 0.
\end{equation}
As a first consequence, we have for a random test point $\Xbf$
\begin{align}\label{eq: convergenceinprobRandomTestpoint}
    \norm{\mu_n(\Xbf)- \mu(\Xbf)}_\H \stackrel{p}{\to} 0,
\end{align}
as 
\begin{align*}
    \Prob( \norm{\mu_n(\Xbf)- \mu(\Xbf)}_\H > \varepsilon ) \leq  \E[\Prob( \norm{\mu_n(\Xbf)- \mu(\Xbf)}_\H > \varepsilon \mid \Xbf )] \to 0,
\end{align*}
by boundedness of $\Prob( \norm{\mu_n(\xbf)- \mu(\xbf)}_\H > \varepsilon \mid \Xbf=\xbf )$.
Crucially, since $\sup_{\ybf_1, \ybf_2}k(\ybf_1, \ybf_2 ) \leq C$ by assumption \kernelass{1}, $\norm{\mu_n(\Xbf)- \mu(\Xbf)}_\H $ is bounded as well. Indeed, $\mu_n(\Xbf)$ in \eqref{eq_app:finalestimator} can also be written as a convex combination of $k(\Ybf_i, \cdot)$,
\begin{align*}
    \mu_n(\Xbf)=\sum_{i=1}^n w_i(\Xbf) k(\Ybf_i, \cdot). 
\end{align*}
Thus, for all $x \in [0,1]$, $n \in \mathbb{N}^{\star}$,
\begin{align*}
    \norm{\mu_n(\Xbf)- \mu(\Xbf)}_\H &= \norm{ \sum_{i=1}^n w_i(\Xbf) (k(\Ybf_i, \cdot) - \E[k(\Ybf,\cdot) \mid \mathbf{X}] ) }_\H\\
    &\leq \max_{i} \norm{ k(\Ybf_i, \cdot) - \E[k(\Ybf,\cdot) \mid \mathbf{X}]  }_\H\\
    & \leq \max_{i} \norm{k(\Ybf_i, \cdot) }_\H + \E[ \norm{ k(\Ybf,\cdot)}_\H \mid \mathbf{X}]\\
    & \leq 2 \sqrt{C},
\end{align*}
since $\norm{k(\Ybf_i, \cdot) }_\H^2=k(\Ybf_i, \Ybf_i) \leq C$.

As for a bounded random variable convergence in probability implies convergence in expectations, the result follows.
\end{proof}

\projmeanconsistency*

\begin{proof}
    Direct application of Proposition \ref{thm: meanconsistency}.
\end{proof}

\VIconsistency*

\begin{proof}
We recall that for the number of trees $N \to \infty$, $\mu_n$ is defined in \eqref{eq_app:finalestimator}, and
\begin{align*}
   I_n^{(j)} = \frac{ \sum_{i=1}^n \norm{ \mu_n(\Xbf_i') - \hmunbruteXl }_{\H}^2 }{\sum_{i=1}^n \norm{ \mu_n(\Xbf_i') - \frac{1}{n} \sum_{i'=1}^n \mu_n(\Xbf'_{i'}) }_{\H}^2}
\end{align*}

We first study the upper part and show
\begin{claim}
   \begin{align}\label{eq: meanconvergenceupper}
  \E[ |\norm{ \mu_n(\Xbf'_1) - \mu_n(\Xbf'^{(-j)}_1) }_{\H}^2 - \norm{ \mu(\Xbf'_1) - \mu(\Xbf'^{(-j)}_1)}_{\H}^2|] \to 0 . 
\end{align} 
\end{claim}

\begin{claimproof}
First,
\begin{align*}
    &\E[ |\norm{ \mu_n(\Xbf_1') - \mu_n(\Xbf_1'^{(-j)})  }_{\H}^2  - \norm{ \mu(\Xbf'_1) - \mu(\Xbf'^{(-j)}_1)}_{\H}^2| ] = \\
    &\E[ |\langle \mu_n(\Xbf_1') - \mu_n(\Xbf_1'^{(-j)}), \mu_n(\Xbf_1') - \mu_n(\Xbf_1'^{(-j)})  \rangle_{\H} - \norm{ \mu(\Xbf'_1) - \mu(\Xbf'^{(-j)}_1)}_{\H}^2| ] =\\
    & \E[| \langle \mu_n(\Xbf_1') - \mu(\Xbf_1') + \mu(\Xbf_1') - \mu(\Xbf_1'^{(-j)}) + \mu(\Xbf_1'^{(-j)}) - \mu_n(\Xbf_1'^{(-j)}), \mu_n(\Xbf_1') - \mu_n(\Xbf_1'^{(-j)})  \rangle_{\H} -\\
    & \norm{ \mu(\Xbf'_1) - \mu(\Xbf'^{(-j)}_1)}_{\H}^2|]\leq \E[ |\langle \mu_n(\Xbf_1') - \mu(\Xbf_1') , \mu_n(\Xbf_1') - \mu_n(\Xbf_1'^{(-j)})  \rangle_{\H}| ] +\\
        &  \E[ |\langle  \mu(\Xbf_1'^{(-j)}) - \mu_n(\Xbf_1'^{(-j)}), \mu_n(\Xbf_1') - \mu_n(\Xbf_1'^{(-j)})  \rangle_{\H}| ]+ \\
        &\E[ |\langle \mu(\Xbf_1') - \mu(\Xbf_1'^{(-j)}) , \mu_n(\Xbf_1') - \mu_n(\Xbf_1'^{(-j)})  \rangle_{\H} - \norm{ \mu(\Xbf'_1) - \mu(\Xbf'^{(-j)}_1)}_{\H}^2|] 
\end{align*}
Now using that all norms of estimates are uniformly bounded by $K < \infty$ and the Cauchy-Schwarz inequality,
\begin{align*}
    \E[| \langle \mu_n(\Xbf_1') - \mu(\Xbf_1') , \mu_n(\Xbf_1') - \mu_n(\Xbf_1'^{(-j)})  \rangle_{\H}| ] &\leq  \E[ \norm{ \mu_n(\Xbf_1') - \mu(\Xbf_1')}_{\H} ] K \to 0\\
    \E[ |\langle  \mu(\Xbf_1'^{(-j)}) - \mu_n(\Xbf_1'^{(-j)}), \mu_n(\Xbf_1') - \mu_n(\Xbf_1'^{(-j)})  \rangle_{\H}| ] &\leq  \E[ \norm{\mu(\Xbf_1'^{(-j)}) - \mu_n(\Xbf_1'^{(-j)})}_{\H} ] K \to 0,
\end{align*}
by Propositions \ref{thm: meanconsistency} and \ref{thm: proj_meanconsistency}. Moreover, 
\begin{align*}
    &\E[ |\langle \mu(\Xbf_1') - \mu(\Xbf_1'^{(-j)}) , \mu_n(\Xbf_1') - \mu_n(\Xbf_1'^{(-j)})  \rangle_{\H} - \norm{ \mu(\Xbf'_1) - \mu(\Xbf'^{(-j)}_1)}_{\H}^2|]= \\
    & \E[ |\langle \mu(\Xbf_1') - \mu(\Xbf_1'^{(-j)}) , \mu_n(\Xbf_1') - \mu_n(\Xbf_1'^{(-j)})  - (\mu(\Xbf_1') - \mu(\Xbf_1'^{(-j)}) ) \rangle_{\H}|  ],
\end{align*}
since again from Propositions \ref{thm: meanconsistency} and \ref{thm: proj_meanconsistency},
\begin{align*}
    &\E[| \langle \mu(\Xbf_1') - \mu(\Xbf_1'^{(-j)}) , \mu_n(\Xbf_1') - \mu_n(\Xbf_1'^{(-j)})  - (\mu(\Xbf_1') - \mu(\Xbf_1'^{(-j)}) ) \rangle_{\H} | ] \leq \\
    &K \E[ \norm{ \mu_n(\Xbf_1') - \mu_n(\Xbf_1'^{(-j)})  - \mu(\Xbf_1') + \mu(\Xbf_1'^{(-j)})}_{\H}] \leq \\
    &K \left( \E[ \norm{ \mu_n(\Xbf_1')   - \mu(\Xbf_1')}_{\H}] + \E[ \norm{ \mu(\Xbf_1'^{(-j)})-\mu_n(\Xbf_1'^{(-j)})}_{\H}] \right) \to 0.
\end{align*}
Thus, we have that \eqref{eq: meanconvergenceupper} holds. 

\end{claimproof}

Given \eqref{eq: meanconvergenceupper}, we can now show that

\begin{claim}
   \begin{align}\label{eq: convergencewewant}
      \E\left[\left | \frac{1}{n} \sum_{i=1}^n \norm{ \mu_n(\Xbf_i') - \hmunbruteXl }_{\H}^2 -\E[\norm{ \mu(\Xbf_1) - \mu(\Xbf_1^{(-j)})}_{\H}^2] \right | \right] \to 0.
 \end{align}
  
\end{claim}

\begin{claimproof}
First, 
    \begin{align*}
     &\E\left[\left | \frac{1}{n} \sum_{i=1}^n \norm{ \mu_n(\Xbf_i') - \hmunbruteXl }_{\H}^2 -\E[\norm{ \mu(\Xbf_1) - \mu(\Xbf_1^{(-j)})}_{\H}^2] \right | \right]\leq \\
     &\E\left[\left | \frac{1}{n} \sum_{i=1}^n \norm{ \mu_n(\Xbf_i') - \hmunbruteXl }_{\H}^2 - \frac{1}{n} \sum_{i=1}^n \norm{ \mu(\Xbf_i')  - \mu(\Xbf_i'^{(-j)})}_{\H}^2 \right | \right] +\\
     &\E\left[\left | \frac{1}{n} \sum_{i=1}^n \norm{ \mu(\Xbf_i') - \mu(\Xbf_i'^{(-j)}) }_{\H}^2 -\E[\norm{ \mu(\Xbf_1) - \mu(\Xbf_1^{(-j)})}_{\H}^2] \right | \right]
\end{align*}
For the first term, by the triangle inequality,
\begin{align*}
    &\E\left[\left | \frac{1}{n} \sum_{i=1}^n \norm{ \mu_n(\Xbf_i') - \hmunbruteXl }_{\H}^2 - \frac{1}{n} \sum_{i=1}^n \norm{ \mu(\Xbf_i')  - \mu(\Xbf_i'^{(-j)})}_{\H}^2 \right | \right]\\
        &\leq \frac{1}{n} \sum_{i=1}^n \E\left[\left |  \norm{ \mu_n(\Xbf_i') - \hmunbruteXl }_{\H}^2 -\norm{ \mu(\Xbf'_i) - \mu(\Xbf'^{(-j)}_i)}_{\H}^2 \right | \right]\\
    &=\E\left[\left | \norm{ \mu_n(\Xbf_1') - \mu_n(\Xbf_1'^{(-j)}) }_{\H}^2 -  \norm{ \mu(\Xbf_1')  - \mu(\Xbf_1'^{(-j)})}_{\H}^2 \right | \right],
\end{align*}
 the latter goes to zero due to \eqref{eq: meanconvergenceupper}.

For the second term, 
\begin{align*}
    \left | \frac{1}{n} \sum_{i=1}^n \norm{ \mu(\Xbf_i') - \mu(\Xbf_i'^{(-j)}) }_{\H}^2 -\E[\norm{ \mu(\Xbf_1) - \mu(\Xbf_1^{(-j)})}_{\H}^2] \right | \stackrel{p}{\to} 0,
\end{align*}
by the law of large numbers. Since the sequence is again uniformly bounded, the same arguments as in Proposition \eqref{thm: meanconsistency}, also imply 
\[
\E \left[ \left | \frac{1}{n} \sum_{i=1}^n \norm{ \mu(\Xbf_i') - \mu(\Xbf_i'^{(-j)}) }_{\H}^2 -\E[\norm{ \mu(\Xbf_1) - \mu(\Xbf_1^{(-j)})}_{\H}^2] \right |\right ] \stackrel{p}{\to} 0,
\]
proving the claim.
\end{claimproof}

In turn, \eqref{eq: convergencewewant} implies weak convergence of the upper part of $I_n^{(-j)}$, that is,
\begin{align*}
    \frac{1}{n} \sum_{i=1}^n \norm{ \mu_n(\Xbf_i') - \hmunbruteXl }_{\H}^2  \stackrel{p}{\to} \E[\norm{ \mu(\Xbf_1) - \mu(\Xbf_1^{(-j)})}_{\H}^2].
\end{align*}

For the lower part, we first proof

\begin{claim}
   \begin{align}\label{condition}
\E[\| \frac{1}{n} \sum_{i=1}^n \mu_n(\Xbf_i') - \E[\mu(\Xbf_1') ] \|_{\H}] \to 0.   
\end{align} 
\end{claim}

\begin{claimproof}
    To show \eqref{condition}, note that
\begin{align*}
    &\E[ \norm{\frac{1}{n} \sum_{i=1}^n \mu_n(\Xbf_i') - \E[\mu(\Xbf_1')]  }_{\H}] \leq  \\
    &\E[  \norm{ \frac{1}{n} \sum_{i=1}^n \mu_n(\Xbf_i') -  \frac{1}{n} \sum_{i=1}^n  \mu(\Xbf_i')  }_{\H}] +\E[ \norm{ \frac{1}{n} \sum_{i=1}^n \mu(\Xbf_i')   - \E[\mu(\Xbf_1')] }_{\H}] 
\end{align*}

For the first term,
\begin{align*}
   \E[  \norm{ \frac{1}{n} \sum_{i=1}^n \mu_n(\Xbf_i') -  \frac{1}{n} \sum_{i=1}^n  \mu(\Xbf_i')  }_{\H}]  &\leq  \frac{1}{n} \sum_{i=1}^n \E[ \norm{\mu_n(\Xbf_i') - \mu(\Xbf_i')  }_{\H}] \\
   &=\E[ \norm{\mu_n(\Xbf_1) - \mu(\Xbf_1)  }_{\H}] \to 0.
\end{align*}
Moreover, it follows by the Law of Large numbers on $\H$ (see e.g., \citet[Chapter 7]{hilbertspacebook}), that
\begin{align}
    \norm{ \frac{1}{n} \sum_{i=1}^n \mu(\Xbf_i')   - \E[\mu(\Xbf_1') }_{\H} \stackrel{p}{\to} 0.
\end{align}
As $\norm{ \frac{1}{n} \sum_{i=1}^n \mu(\Xbf_i')   - \E[\mu(\Xbf_1') }_{\H}$ is uniformly bounded, we have 
\[
\E[ \norm{ \frac{1}{n} \sum_{i=1}^n \mu(\Xbf_i')   - \E[\mu(\Xbf_1')] }_{\H}] \to 0.
\]

\end{claimproof}

Thus, $\frac{1}{n} \sum_{i=1}^n \mu_n(\Xbf_i')$ is a mean-consistent estimate of $\E[\mu(\Xbf_1')]$ and the same argument to prove \eqref{eq: convergencewewant} applied again shows that 
\begin{align*}
    \frac{1}{n} \sum_{i=1}^n \norm{ \mu_n(\Xbf_i') - \frac{1}{n} \sum_{i'=1}^n \mu_n(\Xbf'_{i'}) }_{\H}^2 \stackrel{p}{\to} \E[ \norm{ \mu(\Xbf_1) - \E[ \mu(\Xbf_1)]}_{\H}^2] = \VH[\mu(\Xbf_1)].
\end{align*}

\end{proof}

We now consider the consistency of the projected DRF,
\begin{align}\label{finalestimator2_projected}
    \hmunprojx = \binom{n}{s_n}^{-1}  \sum_{i_1 < i_2 < \ldots < i_{s_n}} \E_{\varepsilon} \left[ T^{(-j)}(\xbf^{(-j)}, \varepsilon; \Zbf_{i_1}, \ldots, \Zbf_{i_{s_n}}) \right],
\end{align}
where the sum is taken over all $\binom{n}{s_n}$ possible subsamples $\Zbf_{i_1}, \ldots, \Zbf_{i_{s_n}}$ of $\Zbf_{1}, \ldots, \Zbf_{n}$ and $s_n \to \infty$ with $n$ and where
\begin{align*}
    T^{(-j)}(\xbf^{(-j)}, \varepsilon; \Zbf_{1}, \ldots, \Zbf_{s_n})=\sum_{i=1}^{s_n} \frac{\1(\Xbf_{i}^{(-j)} \in \mathcal{L}^{(-j)}(\xbf^{(-j)}))}{|\mathcal{L}^{(-j)}(\xbf^{(-j)})|} k(\Ybf_i, \cdot ).%\Phi(\delta_{\Ybf_{j}}). 
\end{align*}
For simplicity we write here the sum from $j=1, \ldots, s_n$, though it should be understood that $\1(\Xbf_{i}^{(-j)} \in \mathcal{L}^{(-j)}(\xbf^{(-j)}))=0$ for $i$ that are used for tree building and not to populate the leaves, according to \forestass{1}.

It should be noted that \eqref{finalestimator2_projected} is not the same as fitting a forest on the data $((\Ybf_1, \Xbf^{(-j)}_1), \ldots , (\Ybf_n, \Xbf^{(-j)}_n) )$, as growing a tree includes $X_j$ implicitly, while populating the leaves or predicting does not. Nonetheless, key assumptions about the estimator translate from $\hmun$ to $\hmunprojx$:

\begin{enumerate}[label=(\textbf{F\arabic*}')]
    \item\label{forestass1j} (\textit{Honesty}) The data used for constructing $T^{(-j)}$ is split into two halves; the first is used for determining the splits and the second for populating the leaves and thus for estimating the response. The covariates in the second sample may be used for the splits, to enforce the subsequent assumptions, but not the response.
    \item\label{forestass2j} (\textit{Random-split}) At every split point and for all feature dimensions $l \in \{1,\ldots,p\} \setminus j$, the probability that the split occurs along the feature $X_l$ is bounded from below by $\pi/p$ for some $\pi > 0$.
    \item\label{forestass3j} (\textit{Symmetry}) The (randomized) output of $T^{(-j)}$ does not depend on the ordering of the training samples.
    \item\label{forestass5j}(\textit{Data sampling}) 
    To grow $T^{(-j)}$, a subsample of size $s_n$ out of the $n$ training data points is sampled. We consider $s_n=n^{\beta}$ with $0 < \beta < 1$.
\end{enumerate}

\begin{lemma}
    \forestass{1}, \forestass{2}, \forestass{3}, \forestass{5} for $\hmun$ imply respectively \ref{forestass1j}, \ref{forestass2j}, \ref{forestass3j} and \ref{forestass5j} for $\hmunprojx$.
    \end{lemma}
\begin{proof}
   \ref{forestass1j}--\ref{forestass5j} are simply restatements of \forestass{1}--\forestass{3}, \forestass{5}, with tree replaced by $T^{(-j)}$. As each is not impacted by the projection, they continue holding for the projected DRF.
\end{proof}

In particular, the conditional independence statements derived from honesty \forestass{1}, crucial for the proofs in \citet{DRF-paper}, remain the same. Moreover, $T^{(-j)}$ is still a weighted mean involving $k(\ybf_i, \cdot)$. As such most results follow in exactly the same way as in \citet{DRF-paper} and are thus mostly stated for completeness. Throughout we assume that expectations on $\H$ are well-defined. In particular, without always explicitly stating it we assume \kernelass{1} holds, such that $\H$ is separable and measurability issues do not arise, as in \citet{DRF-paper, näf2023confidence}.

Applying the decomposition in \citet[Lemma 9]{DRF-paper} to $\hmunprojx$, we obtain

\begin{lemma}\label{variancebound}
Assume $T(\xbf, \varepsilon; \Zcal_{s_n})$ satisfies \forestass{3}. Then,
\begin{align}
   \VH(\hmunprojx)
   &\leq \left( \frac{s_n}{n} + \frac{s_n^2}{n^2} \right) \VH(T^{(-j)}),
\end{align}
\end{lemma}

\begin{proof}
Since \forestass{3} implies \ref{forestass3j}, the proof of this result is analogous to the one of Lemma 10 in~\citet{DRF-paper}, using the ANOVA decomposition in Lemma \citet[Lemma 9]{DRF-paper}.
\end{proof}

We then need a previous result, which we restate for convenience:

\begin{lemma} \label{lemma2}
Let $T$ be a tree satisfying \forestass{2} and \forestass{4} 
that is trained on data $\Zcal_{s_n}$. Suppose that assumption~\dataass{1} holds for $\Xbf_1,\ldots, \Xbf_{s_n}$. Then,
\begin{align}
    \P \left( \text{\textup{diam}} (\Lcal^{(-j)}(\xbf^{(-j)})) \geq \sqrt{p} \left( \frac{s_n}{2k-1}\right) ^{-0.51 \frac{\log((1-\alpha)^{-1})}{\log(\alpha^{-1})} \frac{\pi}{p}}   \right) \leq p \left( \frac{s_n}{2k-1}\right) ^{-1/2 \frac{\log((1-\alpha)^{-1})}{\log(\alpha^{-1})} \frac{\pi}{p}}.
\end{align}
\end{lemma}

\begin{proof}
   Using Lemma 2 of~\citet{wager2017estimation} for $\text{\textup{diam}} (\Lcal(\xbf))$ and using the fact that,
   \begin{align*}
      \text{\textup{diam}} (\Lcal^{(-j)}(\xbf^{(-j)}))\leq  \text{\textup{diam}} (\Lcal(\xbf)),
   \end{align*}
   see e.g., \citet[Proof of Lemma 6]{Sobol_MDA}, gives the result.
\end{proof}

\begin{lemma}\label{helperlemma}
Let $T$ be a tree satisfying \forestass{1} and \forestass{5}.
Then,
\begin{align}\label{star1star}
    \E[T^{(-j)}(\Zcal_{s_n})] = \E[ \E[k(\Ybf,\cdot) \mid \Xbf^{(-j)} \in \Lcal^{(-j)}(\xbf^{(-j)}) ] ].
\end{align}
\end{lemma}

\begin{proof}
    As \forestass{1} and \forestass{5} imply \ref{forestass1j} and \ref{forestass5j} the proof is completely analogous to the proof of \citet[Lemma 12]{DRF-paper}.\footnote{Though $d$ needs to be exchanged by $p$, a typo in the original paper.}
\end{proof}

\begin{corollary}\label{bias}
In addition to the conditions of Lemma~\ref{lemma2}, assume \dataass{2} and that the trees $T(\xbf, \varepsilon; \Zcal_{s_n})$ in the forest satisfy \forestass{1}. Then, we have
\begin{equation}\label{biasbound}
    \| \E[\hmunprojx] - \mu(\xbf^{(-j)}) \|_{\H} = \O\left( s_n^{-1/2 \frac{\log((1-\alpha)^{-1})}{\log(\alpha^{-1})} \frac{\pi}{p}}\right).
\end{equation}
\end{corollary}

\begin{proof}
    Again the proof follows the exact same steps as in \citet[Corollary 13]{DRF-paper}, using the fact that $\Lcal^{(-j)}(\xbf^{(-j)})$ gets smaller in all dimensions from Lemma~\ref{lemma2} and \eqref{star1star} (recalling that $\mu(\xbf^{(-j)})=\E[k(\Ybf, \cdot) \mid \Xbf^{(-j)}=\xbf^{(-j)}]$).
\end{proof}

\mujconsistency*

\begin{proof}
Again the proof works in the same way as the proof of \citet[Theorem 2]{DRF-paper}, but since the result is more important than the previous ones, we state it here for completeness:

We first note that \kernelass{1} implies $\VH(T^{(-j)}) < \infty$. Thus, from Markov's inequality and Lemma \ref{variancebound},
\begin{align*}
    \P\left( n^{\gamma} \norm{\hmunprojx - \E[ \hmunprojx]}_{\mathcal{H}}  > \varepsilon \right)\leq \frac{n^{2\gamma}}{\varepsilon^2} (s/n + s^2/n^2) \VH(T^{(-j)}) =\frac{1}{\varepsilon^2} \mathcal{O}(n^{2\gamma+\beta-1}),
\end{align*}
where the last step followed from \forestass{5}.
Thus 
\[
n^{\gamma} ||\hmunprojx - \E[ \hmunprojx]||_{\mathcal{H}} =\mathcal{O}_p(1),
\]
for $\gamma \leq (1- \beta)/2$. In particular, it goes to zero for any $\varepsilon > 0$, if $\gamma < (1- \beta)/2$. Since,
\[
n^{\gamma} \left\| \hmunprojx -  \mu(\mathbf{x}) \right\|_{\mathcal{H}} \leq  n^{\gamma}\left\| \hmunprojx -  \E[\hmunprojx] \right\| + n^{\gamma}\left\| \E[\hmunprojx] -  \mu(\mathbf{x}^{(-j)}) \right\|_{\mathcal{H}}, 
\]
the result follows as soon as the second expression goes to zero. Now from Theorem \ref{bias}, with $C_{\alpha}=\frac{\log((1-\alpha)^{-1})}{\log(\alpha^{-1})}$,
\[
n^{\gamma}\| \E[\hmunprojx] - \mu(\mathbf{x}^{(-j)}) \|_{\mathcal{H}} = \mathcal{O}\left( n^{\gamma} s_n^{-1/2 C_{\alpha}  \frac{\pi}{p}}\right)= \mathcal{O}\left( n^{\gamma-1/2 \beta C_{\alpha} \frac{\pi}{p}}\right).
\]
This is bounded provided that,
\begin{align*}
    1/2 \beta C_{\alpha} \frac{\pi}{p} \geq  \gamma.
\end{align*}
This proves convergence in probability. Using \kernelass{1} convergence in expectation follows in the same way as argued in Proposition \ref{thm: meanconsistency}.
\end{proof}

Finally, given Proposition \ref{thm: meanconsistencyprojected}, Theorem 2 can be proven with the same arguments as in Theorem \ref{thm: VIconsistency}.

\vfill

\end{document}